\documentclass[lettersize,journal]{IEEEtran}
\usepackage{amsmath,amsfonts}
\usepackage{amsthm}
\usepackage{algorithmic}
\usepackage{algorithm}
\usepackage{array}
\usepackage[caption=false,font=normalsize,labelfont=sf,textfont=sf]{subfig}
\usepackage{textcomp}
\usepackage{stfloats}
\usepackage{url}
\usepackage{verbatim}
\usepackage{graphicx}
\usepackage{cite}
\usepackage{color}

\newtheorem{definition}{Definition}
\newtheorem{proposition}{Proposition}

\usepackage{amssymb}
\usepackage{multirow}
\usepackage{color}
\usepackage {epsfig}
\usepackage{pifont}

\begin{document}

\title{\textit{Fuzziness-tuned:} Improving the Transferability of Adversarial Examples}

	\author{\IEEEauthorblockN{Xiangyuan Yang\IEEEauthorrefmark{2},  Jie Lin\IEEEauthorrefmark{2}\IEEEauthorrefmark{1}, Hanlin Zhang\IEEEauthorrefmark{3}, Xinyu Yang\IEEEauthorrefmark{2}, and Peng Zhao\IEEEauthorrefmark{2}} \\
	\IEEEauthorblockA{\IEEEauthorrefmark{2} Xi'an Jiaotong University, Xi'an, China,\\
		Emails: ouyang\_xy@stu.xjtu.edu.cn, \{jielin, yxyphd,p.zhao\}@mail.xjtu.edu.cn} \\
	\IEEEauthorblockA{\IEEEauthorrefmark{3} Qingdao University, Qingdao, China, Email:
		hanlin@qdu.edu.cn}
	\thanks{Corresponding author: Jie Lin (jielin@mail.xjtu.edu.cn).}
	\thanks{Manuscript received XXX, XX, 2022; revised XXX, XX, 2022.}}

\markboth{Journal of \LaTeX\ Class Files, ~Vol.~XX, No.~XX, XXX~2022}%
{Yang \MakeLowercase{\textit{et al.}}: \textit{Fuzziness-Tuned}:  Improving the Transferability of Adversarial Examples}



\maketitle

\begin{abstract}
With the development of adversarial attacks, adversairal examples have been widely used to enhance the robustness of the training models on deep neural networks. Although considerable efforts of adversarial attacks on improving the transferability of adversarial examples have been developed, the attack success rate of the transfer-based attacks on the surrogate model is much higher than that on victim model under the low attack strength (e.g., the attack strength $\epsilon=8/255$). In this paper, we first systematically investigated this issue and found that the enormous difference of attack success rates between the surrogate model and victim model is caused by the existence of a special area (known as fuzzy domain in our paper), in which the adversarial examples in the area are classified wrongly by the surrogate model while correctly by the victim model. 
Then, to eliminate such enormous difference of attack success rates for improving the transferability of generated adversarial examples, a fuzziness-tuned method consisting of confidence scaling mechanism and temperature scaling mechanism is proposed to ensure the generated adversarial examples can effectively skip out of the fuzzy domain. The confidence scaling mechanism and the temperature scaling mechanism can collaboratively tune the fuzziness of the generated adversarial examples through adjusting the gradient descent weight of fuzziness and stabilizing the update direction, respectively. Specifically, the proposed fuzziness-tuned method can be effectively integrated with existing adversarial attacks to further improve the transferability of adverarial examples without changing the time complexity. Extensive experiments demonstrated that fuzziness-tuned method can effectively enhance the transferability of adversarial examples in the latest transfer-based attacks, e.g., SINI/VMI-FGSM, FIA and SGM, by up to 12.69\% on CIFAR10, 11.9\% on CIFAR100 and 4.14\% on ImageNet for attacking five naturally trained victim models and up to 6.11\% on ImageNet for attacking eight advanced defense methods, respectively.
\end{abstract}


\section{Introduction}
\label{sec:intro}
\IEEEPARstart{D}{ue} to the linear property of deep neural networks (DNNs)~\cite{Linear-property}, adversarial examples that are generated by adding the imperceptible perturbation into natural examples can be effectively used to disrupt decision making of implemented models in deep neural networks~\cite{Adversarial-examples}. Hence, adversarial examples have been widely developed to explore the vulnerability of implemented models in deep neural networks, aiming to enhance the robustness of these implemented models.

Usually, the architecture and parameters of the victim model that the adversary aims to disrupt cannot be obtained accurately, and thus adversarial examples need to be generated by a surrogate model and then applied to the victim model, which is known as transfer-based black-box attacks. Obviously, the transferability of adversarial examples from the surrogate model to the victim model will tremendously affect the attack success rates of transfer-based black-box attacks~\cite{MI-FGSM}, i.e., the higher the transferability of adversarial examples, the greater the attack success rates of transfer-based black-box attacks.

Recently, considerable efforts on transfer-based black-box attacks have been developed to generate effective adversarial examples only through the surrogate model which does not need any information with respect to the architecture,  parameters and output of the victim model~\cite{FGSM,I-FGSM,MI-FGSM,DI-FGSM,SI-NI-FGSM,VMI-FGSM,TI-FGSM,FIA,NAA,SGM}.  For instance, the fast gradient sign method (FGSM)~\cite{FGSM} and its iterative version (i.e., I-FGSM)~\cite{I-FGSM} were the firstly proposed transfer-based attacks to generate the adversarial examples by the surrogate model when the adversary cannot obtain any information about the victim model. To avoid the adversarial example generated by the surrogate model getting stuck into the bad local optimum, i.e., leading to poor transferability of generated adversarial examples, Dong {\em et al.}~\cite{MI-FGSM} stabilized the update direction by adding the momentum term into the gradient. In addition, through introducing variance tuning, the update direction was further stabilized~\cite{VMI-FGSM}. Lin {\em et al.}~\cite{SI-NI-FGSM} used Nesterov accelerated gradient to improve the convergence significantly. Wang {\em et al.}~\cite{FIA} proposed a feature importance-aware attack (FIA) to corrupt the middle layer features with aggregate gradient, which can eliminate the model-specific information to avoid overfitting. Wu {\em et al.}~\cite{SGM} found that the adversarial examples generated by using more gradient from skip connections can achieve higher transferability than that by the residual modules in ResNet-like model.

\begin{figure}[t]
\begin{center}
\centerline{\includegraphics[width=\columnwidth]{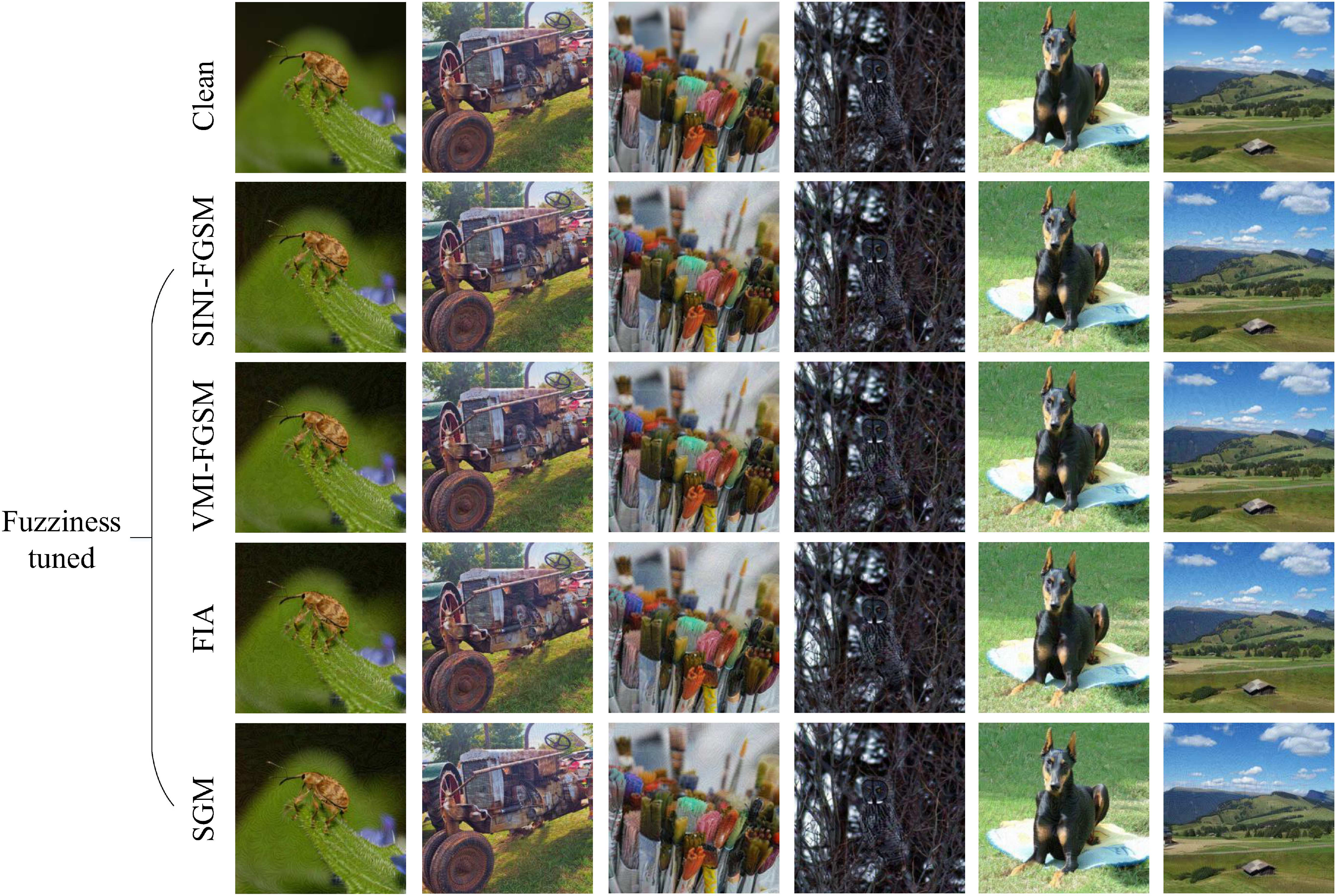}}
\caption{Six clean examples and their adversarial versions generated by fuzziness-tuned attacks.}\vspace{-5mm}
\label{FIG:adversarial_examples}
\end{center}
\end{figure}

However, when these existing transfer-based attacks achieve low attack strength, i.e., $\epsilon=8/255$, the attack success rate (ASR) of generated adversarial examples on the victim model will be much lower than that on the surrogate model, which means the generated adversarial examples achieve poor transferability. For example, based on our results in Section~\ref{sec:experiments}, when the naturally trained models are served as the victim models, the minimum difference of the attack success rates on the surrogate model and the victim model can be achieved as 32\% on ImageNet~\cite{ImageNet} dataset, i.e., the attack success rate on the victim model can be 32\%  lower than that on the surrogate model.


To address this issue, in this paper, we first systematically investigated the reason why such an enormous difference in attack success rates existed between the surrogate model and victim model in transfer-based attacks with low attack strength. Then, a fuzziness-tuned method consisting of the confidence scaling mechanism and temperature scaling mechanism is proposed to eliminate the enormous difference in attack success rates between the surrogate model and victim model in the transfer-based attacks with low attack strength, thereby improving the transferability of generated adversarial examples on the victim model on basis of guaranteeing the great attack success rate on the surrogate model. Additionally, Fig. ~\ref{FIG:adversarial_examples} shows six clean examples and the corresponding adversarial examples generated by our proposed fuzziness-tuned method where we can obviously observe that the adversarial examples generated by our fuzziness-tuned method are imperceptible compared to clean examples.


Our main contributions can be summarized as follows:
\begin{itemize}
\item First, we found that the enormous difference of attack success rates between the surrogate model and victim model in transfer-based attacks with low attack strength is caused by the existence of a special area during adversarial examples generation, in which adversarial examples generated falling in this special area will be classified wrongly on the surrogate model but correctly on the victim model, thereby reducing the transferability of generated adversarial examples from the surrogate model to the victim model. In our paper, the found special area is mathematically defined as \textit{fuzzy domain} and the concept of \textit{fuzziness} is proposed as a metric to indicate whether an adversarial example generated falls in the fuzzy domain or not. 

\item Second, a fuzziness-tuned method that consists of the confidence scaling mechanism and the temperature scaling mechanism is proposed to effectively tune the fuzziness of the generated adversarial examples through adjusting the gradient descent weight of fuzziness and stabilizing the update detection, which can ensure the generated adversarial examples can skip out of the fuzzy domain, thereby enhancing the transferability of generated adversarial examples from the surrogate model to the victim model. Additionally, according to our theoretical analysis, the default cross-entropy (CE) loss used to generate adversarial examples can be a special case of our fuzziness-tuned method with parameter $\mathcal{K}=1$ in the confidence scaling mechanism and parameter $\mathcal{T}=1$ in the temperature scaling mechanism. That means, our proposed fuzziness-tuned method can effectively assist default CE loss in generating adversarial examples with skipping out of fuzzy domain by adjusting the parameters $\mathcal{K}$ and $\mathcal{T}$ as well.


\item Lastly, to evaluate the effectiveness of our proposed fuzziness-tuned method on improving the transferability of generated adversarial examples in the transfer-based attacks with low attack strength, extensive experiments have been conducted with respects to the attack success rates of several latest transfer-based attacks (e.g., SINI/VMI-FGSM, FIA and SGM). The results show that, our proposed fuzziness-tuned method can enhance the transferability of generated adversarial examples by up to 12.69\% on CIFAR10, 11.9\% on CIFAR100 and 4.14\% on ImageNet when attacking against five naturally trained victim models and up to 6.11\% on ImageNet when attacking againt eight advanced defense methods, respectively. 
\end{itemize}

The remainder of the paper is organized as follows: We define the notation and present the basic ideas of several latest transfer-based attacks in Section~\ref{sec:preliminary}. We present  fuzziness-tuned method in Section~\ref{sec:methodology}.  We show experimental results to validate our findings in Section~\ref{sec:experiments}, we review the related works in Section~\ref{sec:related_work} and conclude this paper in Section~\ref{sec:conclusion}, respectively.

\begin{table}
	\centering \caption{Notation}\label{Notation}
	\begin{tabular}{lp{0.31\textwidth}} \hline
		$x:$& The natural example.\\
		$y_o:$& The corresponding ground truth label of natural example $x$.\\
		$f:$& The surrogate model.\\ 
		$L(x,y_o;\theta):$&  The loss function (e.g., the cross-entropy loss) of the surrogate model $f$ with parameters $\theta$.\\
		$x^{adv}:$& The corresponding adversarial example of natural example $x$.\\
		$\epsilon:$& The magnitude of the adversarial perturbations, i.e. the attack strength.\\
		$C:$& The number of categories in the classification task.\\
		$\boldsymbol{z}=[z_1,z_2,\cdots,z_C]:$& The logit vector of the surrogate model $f$ output.\\
		$z_o:$& The logit of the surrogate model $f$ with respect to the ground truth label $y_o$. \\
		\hline
	\end{tabular}
\end{table}

\section{Preliminary}
\label{sec:preliminary}
In this section, the notations used in this paper are defined firstly, and then the basic ideas of several latest transfer-based attacks that our fuzziness-tuned method will be applied with are mentioned.  

\subsection{Notations}

All notations are defined in Table~\ref{Notation}. Note that, in these transfer-based attacks, the adversarial example $x^{adv}$ is generated by maximizing the loss function (e.g., $L(x^{adv}, y_o; \theta)$) in the $l_p$ norm bounded perturbations. Additionally, to be consistent with the existing works~\cite{FGSM,I-FGSM,MI-FGSM,DI-FGSM,SI-NI-FGSM,VMI-FGSM,FIA,SGM}, the parameter $p$  in our paper is set as $\infty$. That is $l_\infty$ norm bounded perturbations are used in our paper to measure the distortion between natural example $x$ and corresponding adversarial example $x^{adv}$, which can be constrained as

\begin{align}
\left\| x-x^{adv} \right\| _\infty \leq \epsilon \quad (i.e.,\quad x^{adv}\in \mathcal{B}(x,\epsilon)), 
\end{align}
where $\epsilon$ is the magnitude of the adversarial perturbations and $\mathcal{B}(x,\epsilon)$ is the spherical neighborhood with $x$ as the center and $\epsilon$ as the radius, respectively.



\subsection{Basic Ideas of Several Latest Transfer-based Attacks}\label{Family_of_FGSM}
The basic ideas of several latest transfer-based attacks including \textbf{FGSM}~\cite{FGSM},  \textbf{I-FGSM}~\cite{I-FGSM}, \textbf{MI-FGSM}~\cite{MI-FGSM}, \textbf{NI-FGSM}~\cite{SI-NI-FGSM}, \textbf{VMI-FGSM}~\cite{VMI-FGSM}, as well as \textbf{FIA}~\cite{FIA}, \textbf{SGM}~\cite{SGM} and \textbf{RCE} loss~\cite{RCE}, are  briefly mentioned as below.

\textbf{FGSM}~\cite{FGSM} was the firstly proposed gradient-based attack, which can generate adversarial examples by maximizing the loss $L(x^{adv},y_o;\theta)$ with one step update on natural examples:
\begin{align}
x^{adv}=x+\epsilon\cdot sign(\nabla_x L(x,y_o;\theta))
\label{equation:fgsm}
\end{align}
where $sign(\cdot)$ is the sign function and $\nabla_x L(x,y_o;\theta)$ is the gradient of the loss function with respect to natural example $x$.

\textbf{Iterative FGSM (I-FGSM)}~\cite{I-FGSM} is an iterative version of FGSM, which can generate adversarial examples by maximizing the loss $L(x^{adv},y_o;\theta)$ with multiple small step updates on natural examples: 
\begin{align}
x^{adv}_{t+1}=Clip^\epsilon_x\{x^{adv}_t+\alpha\cdot sign(\nabla_{x^{adv}_t}L(x^{adv}_t,y_o;\theta))\}
\label{equation:ifgsm}
\end{align}
where $x^{adv}_t$ is the generated adversarial example at the $t$-th step update and $x^{adv}_0$ is the natural example $x$ (i.e., $x^{adv}_0=x$), $\alpha$ represents the step update length, and $Clip^\epsilon_x(\cdot)$ function is used to ensure the generated adversarial examples fall into the $\epsilon$-ball of the natural example $x$.

\textbf{Momentum I-FGSM (MI-FGSM)}~\cite{MI-FGSM} can generate adversarial examples by adding the momentum into the gradient in each step update to avoid the generated adversarial examples getting stuck into the local optimum:
\begin{align}
g_{t+1} &= \mu \cdot g_t + \frac{\nabla_{x^{adv}_t}L(x^{adv}_t,y_o;\theta)}{\|\nabla_{x^{adv}_t}L(x^{adv}_t,y_o;\theta)\|_1} \label{equation:mifgsm_gradient}, \\
x^{adv}_{t+1} &= Clip^{\epsilon}_x\{x^{adv}_t+\alpha\cdot sign(g_{t+1})\} \label{equation:mifgsm_update}
\end{align}
where $g_t$ is the accumulated gradient at the $t$-th iteration ($g_0=0$) and $\mu$ is the decay factor of $g_t$, respectively.

\textbf{Nesterov I-FGSM (NI-FGSM)}~\cite{SI-NI-FGSM} integrates Nesterov Accelerated Gradient (NAG) into I-FGSM to leverage looking ahead property of NAG and prompts generated adversarial examples to escape from poor local maxima easier and faster, thereby improving the transferability of generated adversarial examples. In comparison with MI-FGSM, NI-FGSM adds the looking ahead operation before Equation~(\ref{equation:mifgsm_gradient}):
\begin{align}
x^{nes}_t &= x^{adv}_t+\alpha\cdot \mu \cdot g_t
\label{equation:looking-ahead},\\
g_{t+1} &= \mu \cdot g_t + \frac{\nabla_{x^{adv}_t}L(x^{nes}_t,y_o;\theta)}{\|\nabla_{x^{adv}_t}L(x^{nes}_t,y_o;\theta)\|_1} \label{equation:mifgsm_gradient}, \\
x^{adv}_{t+1} &= Clip^{\epsilon}_x\{x^{adv}_t+\alpha\cdot sign(g_{t+1})\}.
\end{align}

\textbf{Variance tuning MI-FGSM (VMI-FGSM)}~\cite{VMI-FGSM} further stabilizes each step update direction of MI-FGSM through variance tuning, in which  the variance $v_t$ is added to calculate the gradient $g_{t+1}$:
\begin{align}
g_{t+1} &= \mu \cdot g_t + \frac{\nabla_{x^{adv}_t}L(x^{adv}_t,y;\theta)+v_t}{\|\nabla_{x^{adv}_t}L(x^{adv}_t,y;\theta)+v_t\|_1} \label{equation:vmifgsm-gradient}, \\
v_{t+1} &= \frac{\sum\nolimits_{i=1}^N{\nabla _{x^{adv}_{ti}}L(x^{adv}_{ti},y;\theta)}}{N} - \nabla_{x^{adv}_t}L(x^{adv}_t,y;\theta) \label{equation:vmifgsm-variance},\\
x^{adv}_{ti}&=x^{adv}_t+r_i, r_i\sim U[-(\beta\cdot\epsilon)^d,(\beta\cdot\epsilon)^d],
\end{align}
where $N$ is the number of examples,  $v_{t}$ is the gradient variance at the $t$-th step update ($v_0=0$), and $U[a^d,b^d]$ represents the uniform distribution with $d$ dimensions and $\beta$ is a hyperparameter, respectively.

\textbf{Feature importance-aware attack (FIA)}~\cite{FIA} focused on corrupting the middle layer feature with high importance to enhance the transferability by adding a weight to each feature, which can be represented as 
\begin{align}
\Delta^x_k &= \frac{\partial L_{logit}(x,y_o;\theta)}{\partial f_k(x;\theta)} \label{equation:fia_weight} \\
\Delta&=\frac{1}{D}\sum_{n=1}^N{\Delta _{k}^{x\odot M_{p_d}^{n}}}, M_{p_d}\sim Bernoulli(1-p_d) \label{equation:fia_avg_weight},
\end{align}
where $L_{logit}(x,y_o;\theta)$ represents the logit output with respect to the ground truth label $y_o$, $f_k(x;\theta)$ represents the feature maps from the $k$-th layer of the surrogate model $f$, $M_{p_d}$ is a binary matrix with the same size to $x$, $p_d$ is the drop probability,  $\odot$ represents the element-wise product, $D$ is the normalizer that can be obtained by $l_2$-norm on the corresponding summation term, and $N$ indicates the number of random masks applied to the input $x$. In comparison with the family of gradient-based attacks, FIA generates adversarial examples by minimizing the loss function as follows:

\begin{align}
L_{FIA}(x,y_o;\theta) = \sum(\Delta\odot f_k(x;\theta)) \label{equation:fia}.
\end{align}


\textbf{Skip gradient method (SGM)}~\cite{SGM} explored the architectural vulnerability of deep neural networks and found that, in ResNet-like neural networks, adversarial examples generated by using more gradients from the skip connections can achieve higher transferability than that generated by the residual modules.

\textbf{Relative cross-entropy (RCE) loss~\cite{RCE}} is a new normalized CE loss to ensure the generated adversarial examples escape from the poor local maxima through guiding the logit to be updated in the direction of implicitly maximizing its rank distance from the ground-truth class:
%
\begin{align}
Softmax(z_i)=&\frac{e^{z_i}}{\sum\nolimits_{c=1}^C{e^{z_c}}},\label{equation:softmax}\\
L_{CE}(x,y_o;\theta)=&-\log Softmax(z_o),\label{equation:ce}\\\nonumber
L_{RCE}(x,y_o;\theta)=&L_{CE}(x,y_o;\theta)\\
&-\frac{1}{C}\sum\nolimits_{c=1}^CL_{CE}(x,y_c;\theta)\label{equation:rce}
\end{align}
where $z_o$ is the logit of the surrogate model $f$ with respect to the ground truth label $y_o$, and $C$ is the number of categories in the classification task, respectively. 

Although these transfer-based attacks can effectively generate adversarial examples by the surrogate model to the victim model, most of these attacks generate adversarial examples by maximizing the default cross-entropy (CE) loss function and lead to the attack success rates of generated adversarial examples on the surrogate model are much higher than that on the victim model in the transfer-based attacks with low attack strength. That means the generated adversarial examples achieve low transferability from the surrogate model to the victim model. Hence, this calls for a method that can be applied to the transfer-based attacks to eliminate such enormous differences of attack success rates between the surrogate model and the victim model by improving the transferability of generated adversarial examples.


%
\begin{figure}[t]
\begin{center}
\centerline{\includegraphics[width=\columnwidth]{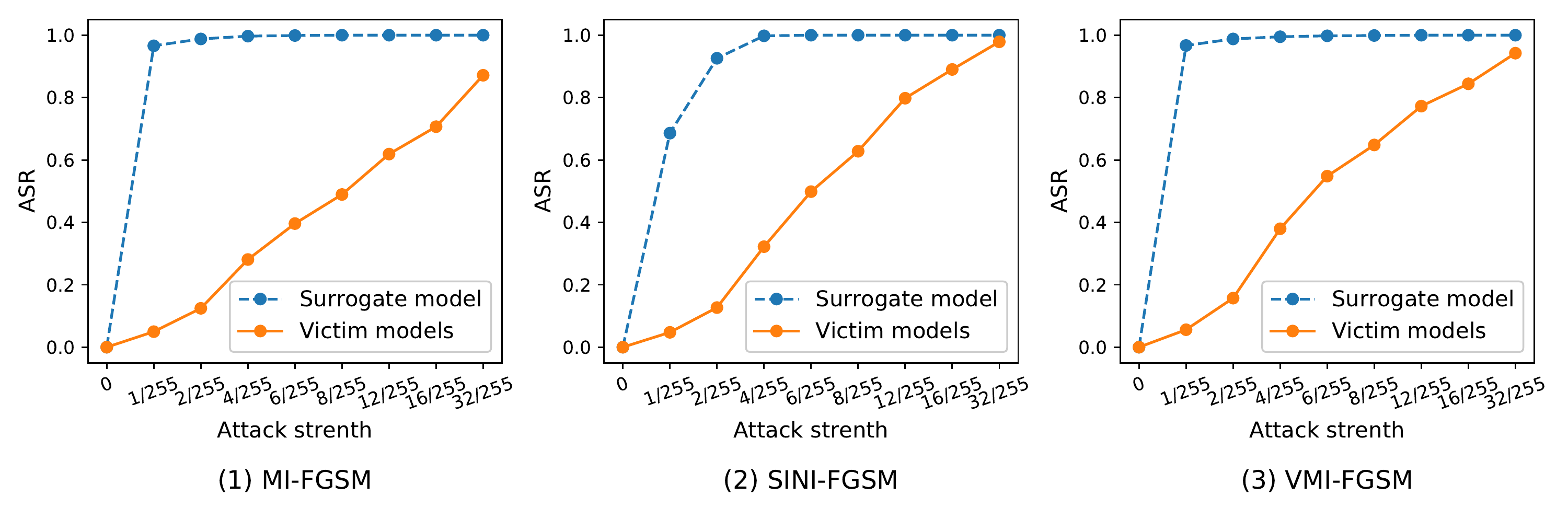}}
\caption{Attack success rate (ASR) curve on the surrogate model and five victim models on ImageNet. Note that, on five victim models, the ASR represents the average value.}\vspace{-5mm}
\label{FIG:ASR_curve_with_different_epsilon}
\end{center}
\end{figure}

\section{Methodology}
\label{sec:methodology}
In this section, we first investigate the reason why the enormous difference of attack success rates between the surrogate model and the victim model exists in transfer-based attacks with low attack strength. Then, a fuzziness-tuned method is proposed to be applied into transfer-based attacks to eliminate such enormous differences of attack success rates and improve the transferability of generated adversarial examples. Finally, the performance of the proposed fuzziness-tuned method is analyzed.


\subsection{The Discovery of Fuzzy Domain}
\label{sec:the_discovery_of_fuzzy_domain}
As shown in Fig.~\ref{FIG:ASR_curve_with_different_epsilon}, we found that, as the attack strength $\epsilon$ increases, the attack success rate (ASR) on the surrogate model can fast converge to 100\%, while the attack success rate (ASR) on five victim models converge slowly to close to 100\%. Specifically, in MI-FGSM~\cite{MI-FGSM}, when the attack strength $\epsilon$ is set as $16/255$, i.e. $\epsilon=16/255$, the ASR on the surrogate model is about $30\%$ higher than that on five victim models. Although SINI/VMI-FGSM~\cite{SI-NI-FGSM, VMI-FGSM} decrease the ASR difference between the surrogate model and five victim models to $11\%$ when the attack strength $\epsilon=16/255$, our investigation found that a large ASR difference, i.e. $35\%$, is still existed when the attack strength is reduced to $\epsilon=8/255$.

Based on the above analysis, we found that in transfer-based attacks with low attack strength, the enormous difference in attack success rates between the surrogate model and victim model appeared due to a special area is existed during adversarial examples generation. Adversarial examples generated falling into this special area will be classified wrongly on the surrogate model but correctly on the victim model, i.e., achieving greater ASR on the surrogate model but lower ASR on the victim model.

To investigate the impact of such a special area on ASRs between the surrogate model and victim model in depth, in this paper the concept of \textit{fuzzy domain} is introduced to represent the found special area and mathematically defined in Definition~\ref{definition:fuzzy_domain}. Additionally, the concept of \textit{fuzziness} is defined as well in Definition~\ref{definition:fuzziness} to determine whether the generated adversarial example falls into the fuzzy domain or not.


\begin{definition}[Fuzzy domain]
\label{definition:fuzzy_domain}
The fuzzy domain $\mathbb{F}(x)$ of an input natural example $x$ in the spherical neighborhood $\mathcal{B}(x,\epsilon)$ is composed of overfitting fuzzy domain (denoted as $\mathbb{F}^+(x)$) and underfitting fuzzy domain (denoted as $\mathbb{F}^-(x)$), which can be represented as
\begin{align}
\left\{ \begin{array}{l}
\mathbb{F} (x) = \mathbb{F} ^+(x) \cup \mathbb{F}^-(x) \\
\mathbb{F}^+(x) = \left\{ \hat{x}|\hat{x}\in \mathcal{B}(x,\epsilon)\land \hat{z}_o<z^+ \right\}\\
\mathbb{F}^-(x) = \left\{ \hat{x}|\hat{x}\in \mathcal{B}(x,\epsilon)\land \hat{z}_o>z^- \right\}\\
\end{array} \right.
\label{equation:fuzzy_domain}
\end{align}
where $\hat{z}_o$ is the logit of the surrogate model $f$ with respect to example $\hat{x}$ whose ground truth label is $y_o$, $z^+$ and $z^-$ is the thresholds for the overfitting and underfitting fuzzy domains, respectively.
\end{definition}

\begin{definition}[Fuzziness]
\label{definition:fuzziness}
The fuzziness of a generated adversarial example is defined as the logit of the surrogate model with respect to the generated adversarial example (i.e., $\hat{z}_o$ in Definition~\ref{definition:fuzzy_domain}) and is used to determine whether the generated adversarial example in the fuzzy domain or not.
\end{definition}

Both of definitions are based on the following assumptions that: the success of the adversarial attacks depends on the degree of damage to the required features of the original category. Hence, the logit of the original category can be determined as fuzziness to measure the degree of damage. In theory, the smaller the fuzziness, the greater the degree of damage. However, due to the structural differences between the surrogate model and victim model, continuously reducing the fuzziness will not increase the degree of damage or even decrease the degree of damage, thus reducing the transferability.

Therefore, according to the Definition~\ref{definition:fuzzy_domain}, the smaller the fuzziness of a generated adversarial example, the higher the probability that the generated adversarial example falls into the overfitting fuzzy domain. Meanwhile, the higher the fuzziness of a generated adversarial example, the higher the probability that the generated adversarial example falls into the underfitting fuzzy domain as well.
Hence, an appropriate fuzziness needs to be determined for generating adversarial examples to ensure the generated adversarial examples skip out of the fuzzy domain, thereby effectively eliminating the enormous difference of attack success rates of the generated adversarial example between the surrogate model and victim model for improving the transferability of the generated adversarial example.

\begin{figure}[t]
	\begin{center}
		\centerline{\includegraphics[width=\columnwidth]{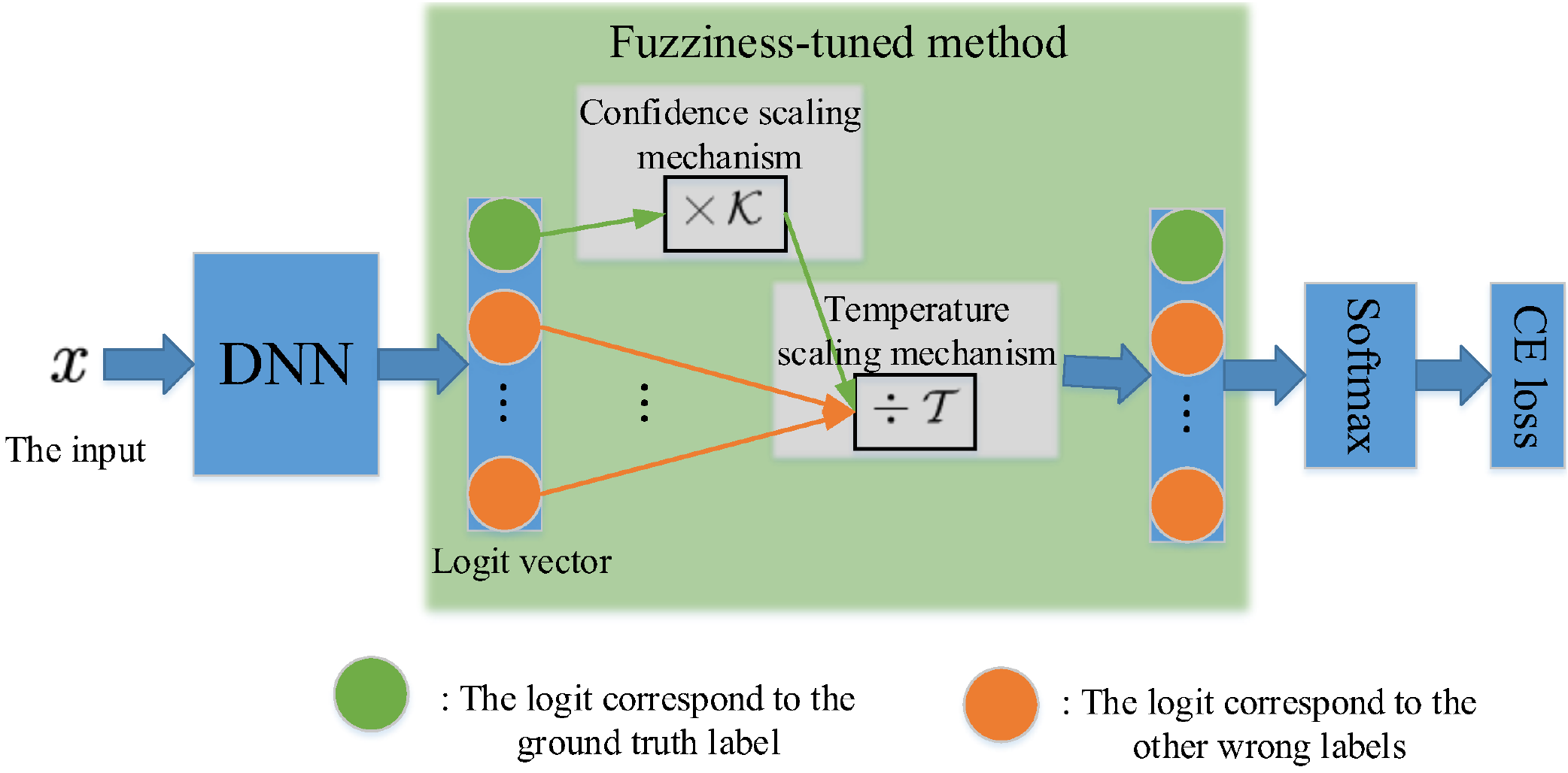}}
		\caption{The framework of fuzziness-tuned method.}\vspace{-5mm}
		\label{FIG:framework_of_FTM}
	\end{center}
\end{figure}

\subsection{Fuzziness-tuned Method}
\label{sec:FTM}

To ensure the generated adversarial examples skip out of the fuzzy domain, a fuzziness-tuned method is proposed to be applied to the latest transfer-based attacks. As shown in Fig.~\ref{FIG:framework_of_FTM}, the proposed fuzziness-tuned method consisting of the confidence scaling mechanism and temperature scaling mechanism is applied to the logit vector of the surrogate model. 

Note that,  the fuzzy domain in adversarial example generation includes the underfitting fuzzy domain and the overfitting domain, which are caused by the insufficient descent of the fuzziness and the excessive descent of the fuzziness on the surrogate model, respectively. Hence, in our paper, the confidence scaling mechanism is proposed to skip the underfitting fuzzy domain by directly increasing the gradient descent weight of fuzziness during adversarial example generation, while the overfitting fuzzy domain is skipped by stabilizing the update direction with the temperature scaling mechanism. Obviously, the confidence scaling mechanism and temperature scaling mechanism can be used collaboratively in the proposed fuzziness-tuned method.

The temperature scaling mechanism firstly was proposed~\cite{Knowledge-distillation} to improve the generalization performance of the student model on knowledge distillation. In our paper, the temperature scaling mechanism is transplanted into our fuzziness-tuned method to have the ability to avoid generated adversarial examples falling into an overfitting fuzzy domain when decreasing the fuzziness by stabilizing the update direction.

In this section, the confidence scaling mechanism and temperature scaling mechanism are described in details, and then the loss function of our fuzziness-tuned method is presented.

\begin{figure*}[ht]
	\begin{center}
		\centerline{\includegraphics[width=\textwidth]{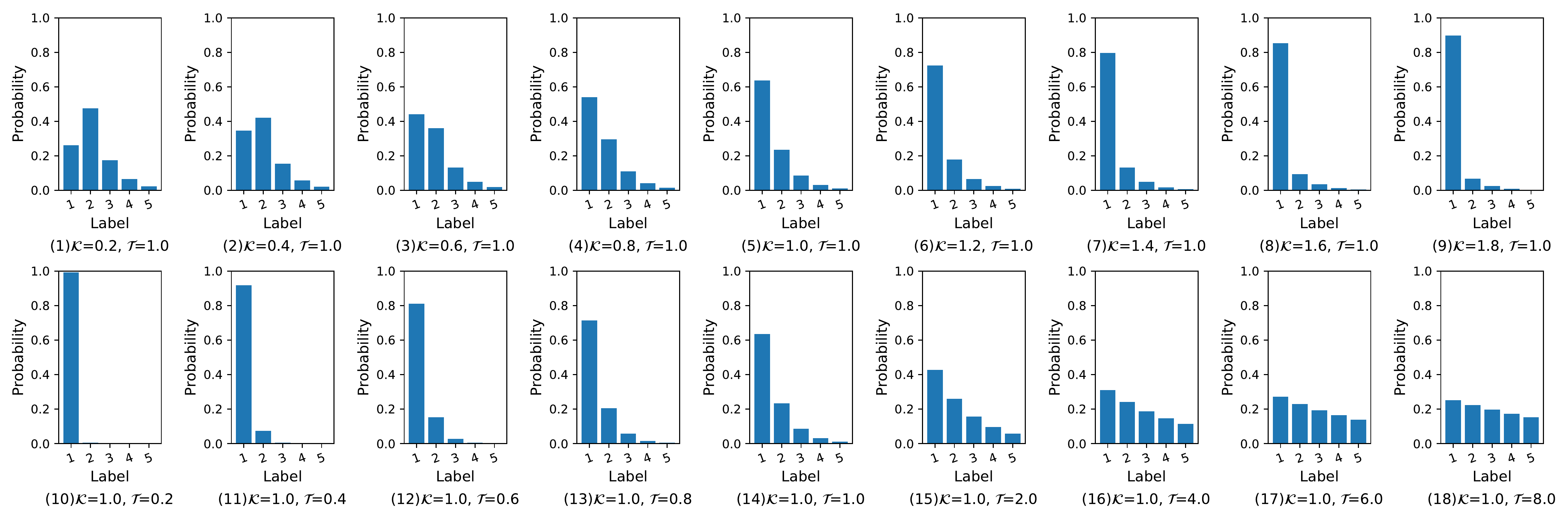}}
		\caption{A case of the probability output of our FSoftmax function with different size of $\mathcal{K}$ and $\mathcal{T}$ where the logit vector is $[2.0, 1.0, 0.0, -1.0, -2.0]$ and the ground truth label is 1. In subfigures (1)-(9), the parameter $\mathcal{T}$ is 1.0 and $\mathcal{K}$ changes from 0.2 to 2.0 with a step size of 0.2. In subfigures (10)-(18), the parameter $\mathcal{K}$ is 1.0 and $\mathcal{T}$ changes from 0.2 to 1.0 with a step size of 0.2 and from 1.0 to 8.0 with a step size of 2.0.}\vspace{-5mm}
		\label{FIG:different_K_T}
	\end{center}
\end{figure*}
\subsubsection{Confidence Scaling Mechanism}
\label{sec:confidence_scaling}

In order to skip the underfitting fuzzy domain, the gradient descent weight of the fuzziness of generated adversarial example needs to be increased for accelerating the decline of the metric. Hence, based on the definition of the fuzziness in Definition~\ref{definition:fuzziness},  \textit{the confidence scaling mechanism}, namely $CSM$, is proposed to directly increase the confidence of the ground truth label through introducing a parameter $\mathcal{K}$ to be muliplied with the logit corresponding to the ground truth label, which can be represented as
\begin{align}
CSM(\boldsymbol{z};\mathcal{K}) = \left\{ \begin{array}{c}
\mathcal{K}\cdot z_i,i=o\\
z_i,i\ne o\\
\end{array} \right.
\end{align}
where $\boldsymbol{z}=[z_1, z_2, \cdots, z_C]$ is the logit vector of the surrogate model $f$ output, $\mathcal{K}$ is the parameter to control the size of the gradient descent weight of fuzziness.



\subsubsection{Temperature Scaling Mechanism}
\label{sec:temperature_scaling}

To avoid the generated adversarial examples falling into an overfitting fuzzy domain when decreasing the fuzziness, \textit{temperature scaling mechanism}, namely $TSM$, which introduces a parameter $\mathcal{T}$ to divide the logit vector, is proposed to tune the fuzziness of generated adversarial examples indirectly through stabilizing the update direction during adversarial example generation and can be indicated as

\begin{align}
TSM(\boldsymbol{z};\mathcal{T})=\frac{\boldsymbol{z}}{\mathcal{T}}
\end{align}
where $\mathcal{T}$ is the parameter to control the stability of update direction.



\subsubsection{Loss Function of Fuzziness-tuned Method}
\label{sec:the_transformation_of_the_CE_loss}
To be consistent with existing works, Softmax function integrating both confidence scaling mechanism and temperature scaling mechanism, namely \textit{FSoftmax}, is proposed in our fuzziness-tuned method and can be represented as
%
\begin{gather}
FSoftmax \left( z_i;\mathcal{T},\mathcal{K} \right) = Softmax(TSM(CSM(\boldsymbol{z};\mathcal{K});\mathcal{T})) \nonumber \\
=\left\{ \begin{array}{c}
\frac{e^{\mathcal{K} \cdot z_o/\mathcal{T}}}{e^{\mathcal{K} \cdot z_o/\mathcal{T}}+\sum\nolimits_{c=1\land c\ne o}^C{e^{z_c/\mathcal{T}}}}, i=o\\
\frac{e^{z_i/\mathcal{T}}}{e^{\mathcal{K} \cdot z_o/\mathcal{T}}+\sum\nolimits_{c=1\land c\ne o}^C{e^{z_c/\mathcal{T}}}}, i\ne o\\
\end{array} \right.\label{equation:fesoftmax}
\end{gather}
where $C$ is the number of categories in the classification task. Note that the parameters $\mathcal{T}$ and $\mathcal{K}$ are greater than 0 (i.e., $\mathcal{T}>0$ and $\mathcal{K}>0$). 

Fig.~\ref{FIG:different_K_T} shows a case of how the probability output of our FSoftmax funciton is influenced by the parameters $\mathcal{K}$ and $\mathcal{T}$, which is beneficial to understand the performance analysis of the fuzziness-tuned method in Section~\ref{sec:performance_analysis}. In subfigures (1)-(9), with increase of the parameter $\mathcal{K}$, the probability of the ground truth label is becoming larger and the probabilities of all wrong labels are decreasing gradually. In subfigures (10)-(18), as the parameter $\mathcal{T}$ increases, the probability of each label is close to $\frac{1}{C}$ (i.e., 0.2) so that the size relationship among probabilities of different categories is constant.

Generally, the adversarial example is generated by maximizing the cross-entropy (CE) loss, which usually uses the standard Softmax function to calculate the probability output. In our paper, to tune the fuzziness of the generated adversarial example, the FSoftmax function defined in Equation~(\ref{equation:fesoftmax})  replaces the standard Softmax function, and fuzziness-tuned cross-entropy (FCE) loss is proposed to replace the default CE loss to generate the adversarial examples, which can be represented as

%
\begin{gather}
L_{FCE}\left( x,y_o;\theta ,\mathcal{T} ,\mathcal{K} \right) =-\log FSoftmax \left( z_o;\mathcal{T} ,\mathcal{K} \right) \nonumber \\
=L_{CCE}\left( x,y_o;\theta ,\mathcal{K} \right) \lor L_{TCE}\left( x,y_o;\theta ,\mathcal{T} \right) \label{equation:fece}
\end{gather}
where $L_{CCE}$ and $L_{TCE}$ 
are the loss based on confidence scaling mechanism and temperature scaling mechanism, respectively.

\subsection{Performance Analysis}
\label{sec:performance_analysis}
In this section, the performance of the confidence scaling mechanism and the temperature scaling mechanism is analyzed, respectively.

\subsubsection{Analysis of The Confidence Scaling Mechanism}
\label{sec:analysis_of_the_confidence_scaling_mechanism}
When $\mathcal{T}=1$, the FCE loss degenerates into confidence scaling-based cross-entropy loss, namely CCE loss, in which only the confidence scaling mechanism is involved in our fuzziness-tuned method:
\begin{align}
\label{equation:cce}
L_{CCE}(x,y_o;\theta,\mathcal{K}) = L_{FCE}(x,y_o;\theta, 1, \mathcal{K})
\end{align}
Proposition~\ref{proposition:CCE_property} demonstrates that, as the parameter $\mathcal{K}$ increases in the confidence scaling mechanism, the gradient descent weight of the fuzziness in adversarial example generation (denoted as $-\frac{\partial z_o}{\partial x}$) can be increased, thereby accelerating the decrease of the fuzziness of generated adversarial examples and ensuring the generated adversarial examples can skip out of the underfitting fuzzy domain and improve the transferability.

\begin{proposition}
	\label{proposition:CCE_property}
	In $\frac{\partial L_{CCE}}{\partial x}$, with the increase of the parameter $\mathcal{K}$, the weight of the item $(-\frac{\partial z_o}{\partial x})$ becomes larger.
\end{proposition}
\begin{proof}
	The derivation formula of $L_{CCE}(x,y_o;\theta,\mathcal{K})$ w.r.t. the input $x$ is:
	\begin{gather}
	\frac{\partial L_{CCE}}{\partial x}=-\frac{\partial L_{CCE}}{\partial z_o}\cdot ( -\frac{\partial z_o}{\partial x} ) +\sum_{i=1( i\ne o )}^C{\frac{\partial L_{CCE}}{\partial z_i}\cdot \frac{\partial z_i}{\partial x}} \nonumber \\
	= \frac{1}{\ln 2}\cdot (\frac{\mathcal{K}\cdot\sum\nolimits_{i=1(i\ne o)}^C{e^{z_i}}}{e^{\mathcal{K} \cdot z_o}+\sum\nolimits_{i=1(i\ne o)}^C{e^{z_i}}}\cdot (-\frac{\partial z_o}{\partial x}) \nonumber \\
	+\sum_{i=1(i\ne o)}^C{\frac{e^{z_i}}{e^{\mathcal{K} \cdot z_o}+\sum\nolimits_{j=1(j\ne o)}^C{e^{z_j}}}\cdot \frac{\partial z_i}{\partial x}}) 
	\label{equation:derivation_cce}
	\end{gather}
	According to Equation~(\ref{equation:derivation_cce}), the weight $w_o$ of the term $(-\frac{\partial z_o}{\partial x})$ is:
	\begin{gather}
	w_o=\frac{\frac{\mathcal{K}\cdot\sum\nolimits_{i=1(i\ne o)}^C{e^{z_i}}}{e^{\mathcal{K} \cdot z_o}+\sum\nolimits_{i=1(i\ne o)}^C{e^{z_i}}}}{\frac{\mathcal{K}\cdot\sum\nolimits_{i=1(i\ne o)}^C{e^{z_i}}}{e^{\mathcal{K} \cdot z_o}+\sum\nolimits_{i=1(i\ne o)}^C{e^{z_i}}}+\sum_{i=1(i\ne o)}^C{\frac{e^{z_i}}{e^{\mathcal{K} \cdot z_o}+\sum\nolimits_{j=1(j\ne o)}^C{e^{z_j}}}}} \nonumber \\
	=\frac{\mathcal{K}}{\mathcal{K}+1}
	\label{equation:wo_cce}
	\end{gather}
	Then, in Equation~(\ref{equation:wo_cce}), with the increase of $\mathcal{K}$, the weight $w_o$ becomes larger. 
\end{proof}


\subsubsection{Analysis of The Temperature Scaling Mechanism}
\label{sec:analysis_of_the_temperature_scaling_mechanism}
When $\mathcal{K}=1$, the FCE loss degenerates into temperature scaling-based cross-entropy loss, namely TCE loss, in which only the temperature scaling mechanism is involved in our fuzziness-tuned method:
\begin{align}
\label{equation:tce}
L_{TCE}(x,y_o;\theta,\mathcal{T}) = L_{FCE}(x,y_o;\theta, \mathcal{T}, 1)
\end{align}

Proposition~\ref{proposition:TCE_property1} verifies that, the increasing of the parameter $\mathcal{T}$ in the temperature scaling mechanism will not impact the weight of the fuzziness, i.e., $-\frac{\partial z_o}{\partial x}$. Meanwhile, Proposition~\ref{proposition:TCE_property2} certifies that, with the increase of the parameter $\mathcal{T}$,  the update direction on generating adversarial examples can be stabilized, thereby further indirectly decreasing the fuzziness of generated adverarial examples and avoiding the generated adversarial examples falling into overfitting fuzzy domain.
 
 \begin{proposition}
 	\label{proposition:TCE_property1}
 	In $\frac{\partial L_{TCE}}{\partial x}$, with the increase of the parameter $\mathcal{T}$, the weight of the item $(-\frac{\partial z_o}{\partial x})$ has not changed.
 \end{proposition}
 \begin{proof}
 	The derivation equation of $L_{TCE}(x,y_o;\theta,\mathcal{T})$ w.r.t. the input $x$ is:
 	\begin{gather}
 	\frac{\partial L_{TCE}}{\partial x}=-\frac{\partial L_{TCE}}{\partial z_o}\cdot ( -\frac{\partial z_o}{\partial x} ) +\sum_{i=1( i\ne o )}^C{\frac{\partial L_{TCE}}{\partial z_i}\cdot \frac{\partial z_i}{\partial x}} \nonumber \\
 	=\frac{1}{\mathcal{T}\cdot \ln 2}\cdot ( ( 1-\frac{e^{z_o/\mathcal{T}}}{\sum\nolimits_{i=1}^C{e^{z_i}/\mathcal{T}}} ) \cdot ( -\frac{\partial z_o}{\partial x} ) \nonumber \\
 	+\sum_{i=1( i\ne o )}^C{\frac{e^{z_i/\mathcal{T}}}{\sum\nolimits_{j=1}^C{e^{z_j/\mathcal{T}}}}\cdot \frac{\partial z_i}{\partial x}} )
 	\label{equation:derivation_tce}
 	\end{gather}
 	Equation~(\ref{equation:derivation_tce}) shows that the weight $w_o$ of the term $(-\frac{\partial z_o}{\partial x})$ is:
 	\begin{gather}
 	w_o = \frac{( 1-\frac{e^{z_o/\mathcal{T}}}{\sum\nolimits_{i=1}^C{e^{z_i}/\mathcal{T}}} )}{( 1-\frac{e^{z_o/\mathcal{T}}}{\sum\nolimits_{i=1}^C{e^{z_i}/\mathcal{T}}} )+\sum_{i=1( i\ne o )}^C{\frac{e^{z_i/\mathcal{T}}}{\sum\nolimits_{j=1}^C{e^{z_j/\mathcal{T}}}}}} =\frac{1}{2}
 	\label{equation:wo_tce}
 	\end{gather}
 	Therefore, in Equation~(\ref{equation:wo_tce}), the weight $w_o$ is not influenced by the parameter $\mathcal{T}$.
 \end{proof}

 \begin{proposition}
 	\label{proposition:TCE_property2}
 	For an input $x$, $\forall x_1, x_2\in \mathcal{B} (x,\epsilon)$, when the parameter $\mathcal{T}\rightarrow +\infty$ in our TCE loss, the angle $<\frac{\partial L_{TCE}(x_1,y_o;\theta,\mathcal{T})}{\partial x}, \frac{\partial L_{TCE}(x_2,y_o;\theta,\mathcal{T})}{\partial x}>$ achieves the minimum upper bound.
 \end{proposition}
 \begin{proof}
 	Because of the uniqueness of the characteristics of different categories, according to Equation~(\ref{equation:derivation_tce}), we assume that $z_1$ is the logit of the ground truth label $y_o$ (i.e., $z_1=z_o$) and $\{-\frac{\partial z_o}{\partial x}, \frac{\partial z_2}{\partial x},\cdots,\frac{\partial z_C}{\partial x}\}$ in $\frac{\partial L_{TCE}(x,y_o;\theta,\mathcal{T})}{\partial x}$ is an orthogonal basis for any $x$.
 	
 	In the orthogonal basis $\{-\frac{\partial z_o}{\partial x}|_{x_1}, \frac{\partial z_2}{\partial x}|_{x_1},\cdots,\frac{\partial z_C}{\partial x}|_{x_1}\}$, we choose a reference vector $\boldsymbol{\mu}=[1,1,\cdots,1]$. Similarly, in the orthogonal basis $\{-\frac{\partial z_o}{\partial x}|_{x_2}, \frac{\partial z_2}{\partial x}|_{x_2},\cdots,\frac{\partial z_C}{\partial x}|_{x_2}\}$, we choose a reference vector $\boldsymbol{\nu}=[1,1,\cdots,1]$. 
 	
 	Hence, the upper bound of the angle $<\frac{\partial L_{TCE}(x_1,y_o;\theta,\mathcal{T})}{\partial x}, \frac{\partial L_{TCE}(x_2,y_o;\theta,\mathcal{T})}{\partial x}>$ is: 
 	\begin{gather}
 	\max<\frac{\partial L_{TCE}(x_1,y_o;\theta,\mathcal{T})}{\partial x}, \frac{\partial L_{TCE}(x_2,y_o;\theta,\mathcal{T})}{\partial x}> \nonumber \\
 	=<\frac{\partial L_{TCE}(x_1,y_o;\theta,\mathcal{T})}{\partial x}, \boldsymbol{\mu}>+<\boldsymbol{\mu},\boldsymbol{\nu}> \nonumber \\
 	+ <\boldsymbol{\nu},\frac{\partial L_{TCE}(x_2,y_o;\theta,\mathcal{T})}{\partial x}>
 	\label{equation:upper_bound_of_angle}
 	\end{gather}
 	where $<\frac{\partial L_{TCE}(x_1,y_o;\theta,\mathcal{T})}{\partial x}, \boldsymbol{\mu}>$ and $<\boldsymbol{\nu},\frac{\partial L_{TCE}(x_2,y_o;\theta,\mathcal{T})}{\partial x}>$, according to Equation~(\ref{equation:derivation_tce}), can be represented as
 	\begin{gather}
 	\mathrm{arc}\cos \frac{[\sum\nolimits_{i=2}^C{p_i}, p_2,\cdots,p_C]\cdot [1,1,\cdots,1]}{\left\|[\sum\nolimits_{i=2}^C{p_i},p_2,\cdots,p_C]\right\|\left\|[1,1,\cdots,1]\right\|} \nonumber \\
 	=\mathrm{arc}\cos \frac{2}{\sqrt{C}}\cdot \frac{1}{\sqrt{1+\frac{p_2^2+p_3^2+\cdots+p_C^2}{(p_2+p_3+\cdots+p_C)^2}}}
 	\label{equation:arccos_tce}
 	\end{gather}
 	where $p_i=\frac{e^{z_i/\mathcal{T}}}{\sum\nolimits_{j=1}^C{e^{z_j/\mathcal{T}}}}$. According to Cauchy inequality, 
 	\begin{gather}
 	\frac{2}{\sqrt{C}}\cdot \frac{1}{\sqrt{1+\frac{p_2^2+p_3^2+\cdots+p_C^2}{(p_2+p_3+\cdots+p_C)^2}}}\leq \frac{2/\sqrt{C}}{\sqrt{1+\frac{p_2^2+p_3^2+\cdots+p_C^2}{(C-1)(p_2^2+p_3^2+\cdots+p_C^2)}}} \nonumber \\
 	\Rightarrow \frac{2}{\sqrt{C}}\cdot \frac{1}{\sqrt{1+\frac{p_2^2+p_3^2+\cdots+p_C^2}{(p_2+p_3+\cdots+p_C)^2}}}\leq \frac{2\sqrt{C-1}}{C} \nonumber \\
 	\Rightarrow 
 	\left\{ \begin{array}{c}
 	<\frac{\partial L_{TCE}(x_1,y_o;\theta,\mathcal{T})}{\partial x}, \boldsymbol{\mu}>\geq \mathrm{arc}\cos \frac{2\sqrt{C-1}}{C} \\
 	<\boldsymbol{\nu},\frac{\partial L_{TCE}(x_2,y_o;\theta,\mathcal{T})}{\partial x}>\geq \mathrm{arc}\cos \frac{2\sqrt{C-1}}{C} \\
 	\end{array} \right..
 	\label{equation:inequality_tce}
 	\end{gather}
 	When $p_2=p_3=\cdots=p_C$, Equation~(\ref{equation:inequality_tce}) gets the equality, thereby the angles $<\frac{\partial L_{TCE}(x_1,y_o;\theta,\mathcal{T})}{\partial x}, \boldsymbol{\mu}>$ and $<\boldsymbol{\nu},\frac{\partial L_{TCE}(x_2,y_o;\theta,\mathcal{T})}{\partial x}>$ achieve the minimum value. Note that the angle $<\boldsymbol{\mu}, \boldsymbol{\nu}>$ is a deterministic value.
 	
 	When the parameter $\mathcal{T}\rightarrow +\infty$, $p_2=p_3=\cdots=p_C=\frac{1}{C}$ (Fig.~\ref{FIG:different_K_T} shows a case). Therefore, in Equation~(\ref{equation:upper_bound_of_angle}), the angle $<\frac{\partial L_{TCE}(x_1,y_o;\theta,\mathcal{T})}{\partial x}, \frac{\partial L_{TCE}(x_2,y_o;\theta,\mathcal{T})}{\partial x}>$ achieves the minimum upper bound $2\mathrm{arc}\cos \frac{2\sqrt{C-1}}{C}+<\boldsymbol{\mu},\boldsymbol{\nu}>$.
 \end{proof}
Additionally, Proposition~\ref{proposition:relationship_between_TCE_and_RCE} shows that when the parameter $\mathcal{T}$ goes to $+\infty$, our TCE loss-based gradient sign attacks (i.e., I/MI/SINI/VMI-FGSM, etc.) can be equivalent to RCE loss-based gradient sign attacks~\cite{RCE}. That is, the RCE loss can be a special case of our temperature scaling mechanism with the parameter $\mathcal{T}$ being $+\infty$. Note that, the relative cross-entropy (RCE) loss~\cite{RCE} was proposed to guide the logit to be updated in the direction of implicitly maximizing the rank distance from the ground truth label during adversarial example generation, which can achieve great performance on the target attacks.

\begin{proposition}
\label{proposition:relationship_between_TCE_and_RCE}
When the parameter $\mathcal{T}\rightarrow +\infty$ in our TCE loss, $sign(\frac{\partial L_{RCE}}{\partial x})\approx sign(\frac{\partial L_{TCE}}{\partial x})$.
\end{proposition}
\begin{proof}
The derivation formula of $L_{RCE}(x,y_o;\theta)$ and $L_{TCE}(x,y_o;\theta,\mathcal{T})$ w.r.t. the input $x$ are Equations~(\ref{equation:derivation_rce}) and (\ref{equation:derivation_tce}), respectively.
\begin{gather}
\frac{\partial L_{RCE}}{\partial x}=-\frac{\partial L_{RCE}}{\partial z_o}\cdot ( -\frac{\partial z_o}{\partial x} ) +\sum_{i=1( i\ne o )}^C{\frac{\partial L_{RCE}}{\partial z_i}\cdot \frac{\partial z_i}{\partial x}} \nonumber \\
=\frac{1}{\ln 2}\cdot ( \frac{C-1}{C}\cdot( -\frac{\partial z_o}{\partial x} ) +\sum_{i=1(i\ne o)}^C{\frac{1}{C}\cdot \frac{\partial z_i}{\partial x}} )
\label{equation:derivation_rce}
\end{gather}
In Equation~(\ref{equation:derivation_tce}), when $\mathcal{T}\rightarrow +\infty$, $\left( 1-\frac{e^{z_o/\mathcal{T}}}{\sum\nolimits_{i=1}^C{e^{z_i/\mathcal{T}}}} \right) \approx \frac{C-1}{C}$ and $\frac{e^{z_i/\mathcal{T}}}{\sum\nolimits_{j=1}^C{e^{z_j/\mathcal{T}}}} \approx \frac{1}{C}$. Then, 
\begin{align}
sign(\frac{\partial L_{RCE}}{\partial x}) \approx sign(\frac{\partial L_{TCE}}{\partial x})
\end{align}
\end{proof}

\section{Experiments}
\label{sec:experiments}

\subsection{Experimental Setup}
\label{sec:experimental_setup}
\textbf{Datasets.} We randomly pick 2000 clean images from CIFAR10/100~\cite{CIFAR} test datasets and 1000 clean images from ImageNet~\cite{ImageNet} validation set, which can be correctly classified by all deep learning models used in each dataset.

\textbf{Models.} Nine naturally trained models are considered in our evaluations, including VGG16~\cite{VGG}, VGG19~\cite{VGG}, ResNet50~\cite{ResNet}, ResNet152~\cite{ResNet}, ResNext50~\cite{ResNext}, WideResNet-16-4 (WRN-16-4)~\cite{WideResNet}, Inception-v3~\cite{Inception-v3}, DenseNet121~\cite{DenseNet} and Mobilenet-v2~\cite{MobileNet}. Six of these models are selected for each dataset as the surrogate and victim models, wherein VGG16 and ResNet50 are determined as the surrogate models respectively on three datasets and ResNet152 is additionally added as the surrogate model on ImageNet. We train the models from scratch on CIFAR10/100 datasets and adopt the pre-trained models in~\cite{pytorch-image-models,torchvision-models} on the ImageNet dataset.

Eight defense methods include six input modifications, which are Resized and Padding (RP)~\cite{RP}, Bit Reduction (Bit-Red)~\cite{Bit-Red}, JPEG compression (JPEG)~\cite{JPEG}, Feature Distillation (FD)~\cite{FD}, Neural Representation Purifier (NRP)~\cite{NRP} and Randomized Smoothing (RS)~\cite{RS}, and two adversarial trained models, which are adversarial Inception-v3 (Inc-v3$^{adv}$) and ensemble adversarial Inception-ResNet-v2~\cite{Ensemble-Adversarial-Training} (IncRes-v2$^{adv}_{ens}$). The six input modifications adopt MobileNet-v2~\cite{MobileNet} as the victim model on ImageNet.

\textbf{Baselines.} Eight baselines are considered, including FGSM~\cite{FGSM}, I-FGSM~\cite{I-FGSM}, DI-FGSM~\cite{DI-FGSM}, MI-FGSM~\cite{MI-FGSM}, SINI-FGSM~\cite{SI-NI-FGSM}, VMI-FGSM~\cite{VMI-FGSM}, FIA~\cite{FIA} and SGM~\cite{SGM}, where the source code for the FGSM family attacks (i..e, first six baselines) comes from the repository (named Torchattacks~\cite{Torchattacks}). Additionally, the RCE loss~\cite{RCE} is combined with these attacks to be considered as the baselines in our evaluations as well.

\textbf{Metric.} In the evaluations, the adversarial examples are generated by one surrogate model to attack against five victim models, and the average attack success rates (ASRs) of attacks on five victim models are considered as the metric to evaluate the transferability of adversarial examples generated by these attacks, in which the adversarial examples are generated by the surrogate model and used to attack against the victim models. Obviously,  the higher the ASR of attacks on the victim models, the higher the transferability of adversarial examples generated by these attacks.

\textbf{Hyper-parameters.} 
Recall that, our fuzziness-tuned method is proposed to improve the transferability of adversarial examples generated by transfer-based attacks with low attack strength. Hence, the attack strength in the evaluations is set as $8/255$ (i.e., $\epsilon=8/255$), which is comparatively low in comparison with $\epsilon=16/255$. The number of steps and step length of all iterative-based attacks (i.e., I/MI/DI/SINI/VMI-FGSM, FIA and SGM) are set as $T=10, \alpha=0.8/255$, respectively. In the momentum-based attacks (i.e.,  MI/SINI/VMI-FGSM, FIA and SGM), the decay factor $\mu$ is set as 1.0. For DI-FGSM, the resize rate and diversity probability are set as 0.9 and 0.5, respectively. For SINI-FGSM, the number of scale copies is set as 5. For VMI-FGSM, the number of examples and the upper bound of neighborhood are set as $N=20, \beta=1.5$, respectively. For FIA, the drop probability and the ensemble number are set as $p_d=0.3, N=30$, respectively, and the chosen feature layers of ResNet50/152 and VGG16 are set as the last layer of the third block and Conv3\_3, respectively. For SGM, the decay factor $\gamma$ is set as 0.2 for ResNet50/152 to reduce the gradient from the residual modules. 

\begin{figure*}[ht]
	\begin{center}
		\centerline{\includegraphics[width=\textwidth]{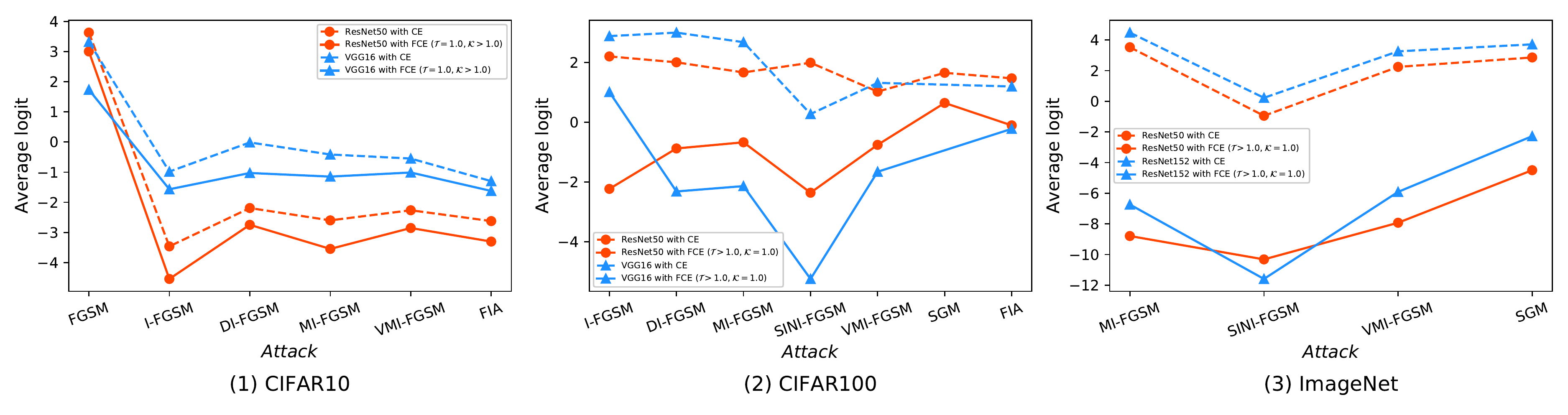}}
		\caption{The average logit of the correct category (or, The average fuzziness) on 2000 CIFAR10/100 test examples (1000 ImageNet test examples) under the untarget attack setting. Different colors of the curve denote different used surrogate models. The solid line denotes using our CCE or TCE loss and the dotted line denotes CE loss.}\vspace{-5mm}
		\label{FIG:average_logit_of_different_attacks}
	\end{center}
\end{figure*}
\begin{figure*}[ht]
	\begin{center}
		\centerline{\includegraphics[width=\textwidth]{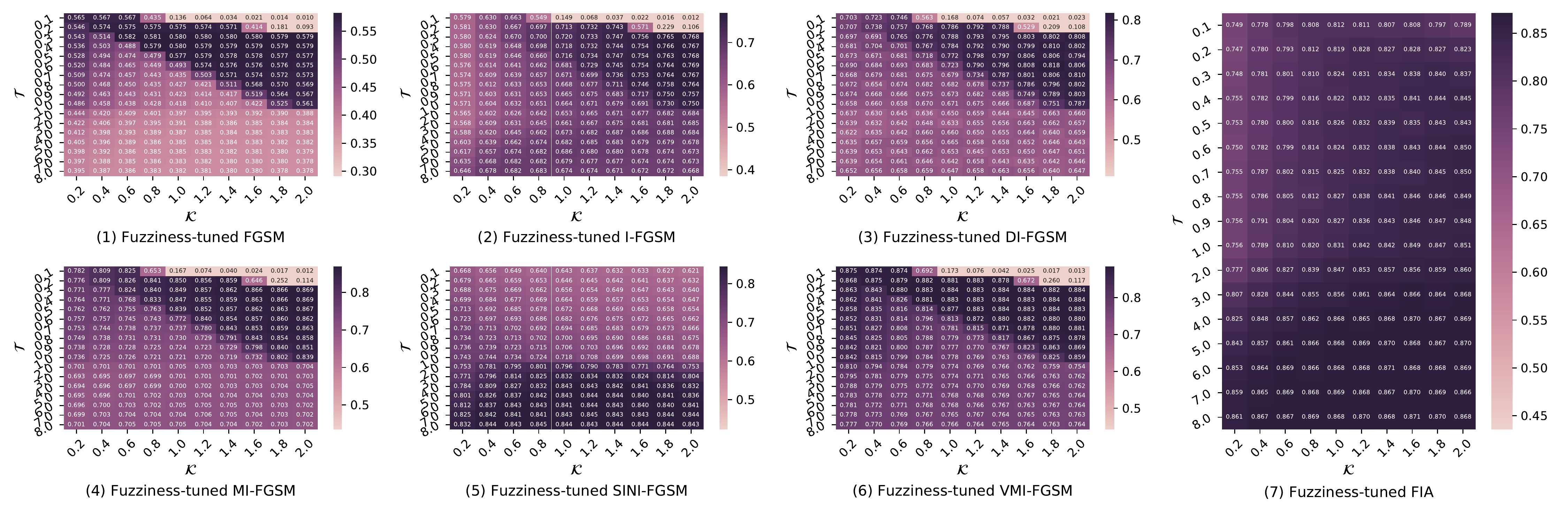}}
		\caption{The average untarget attack success rates (\%) to five victim models with VGG16 as the surrogate model on CIFAR10.}\vspace{-5mm}
		\label{FIG:Ablation_of_K_T_on_CIFAR10_with_VGG16}
	\end{center}
\end{figure*}
\begin{figure*}[ht]
	\begin{center}
		\centerline{\includegraphics[width=\textwidth]{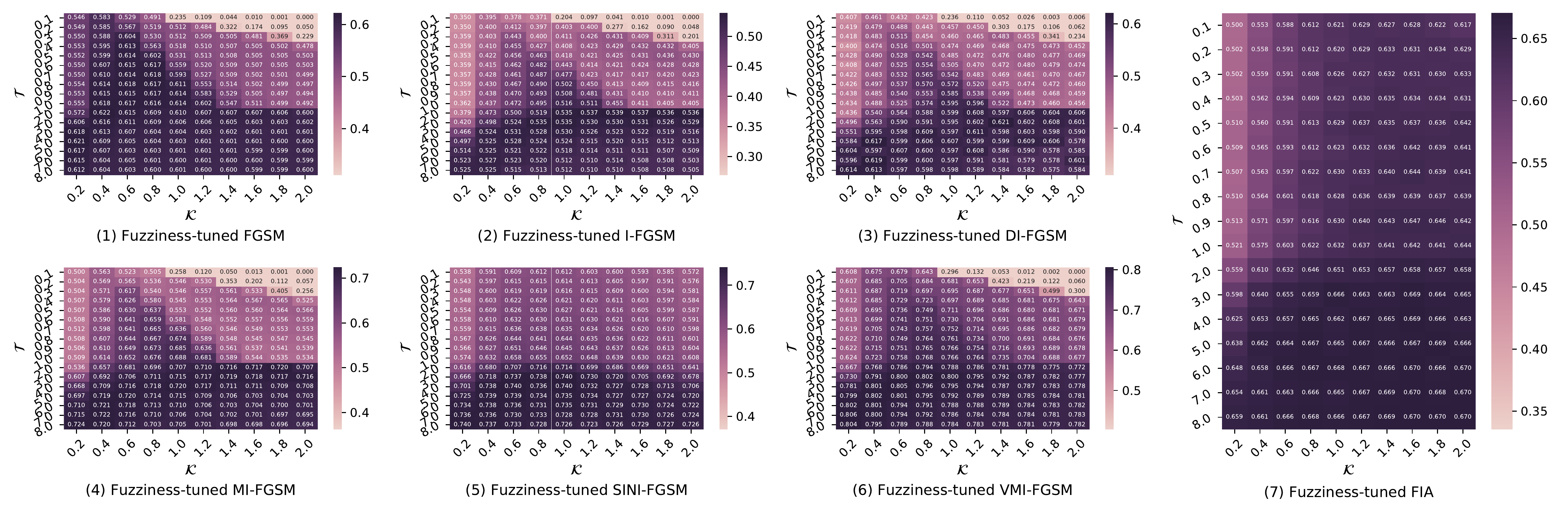}}
		\caption{The average untarget attack success rates (\%) to five victim models with VGG16 as the surrogate model on CIFAR100.}\vspace{-5mm}
		\label{FIG:Ablation_of_K_T_on_CIFAR100_with_VGG16}
	\end{center}
\end{figure*}
\begin{figure*}[ht]
	\begin{center}
		\centerline{\includegraphics[width=\textwidth]{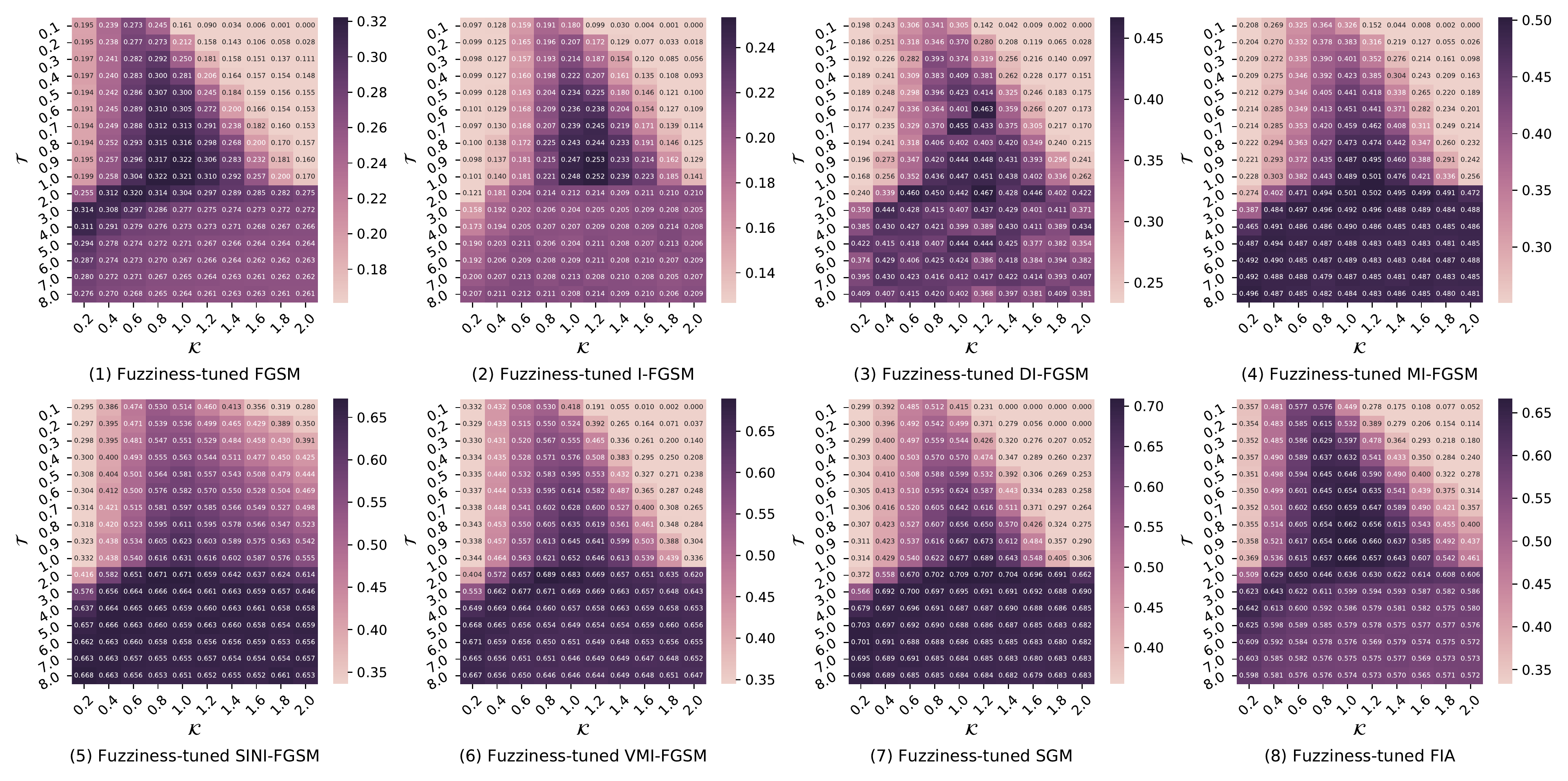}}
		\caption{The average untarget attack success rates (\%) to five victim models with ResNet50 as the surrogate model on ImageNet.}\vspace{-5mm}
		\label{FIG:Ablation_of_K_T_on_ImageNet_with_ResNet50}
	\end{center}
\end{figure*}

\subsection{The Impact of Various Parameter $\mathcal{K}$ and $\mathcal{T}$ on the ASRs of Attacks} 
In this section, the impact of various parameters $\mathcal{K}$ and $\mathcal{T}$ in our fuzziness-tuned method on the average ASRs of attacks against five victim models is evaluated, in which parameter $\mathcal{K}$ varies in $[0.2, 2]$ and parameter $\mathcal{T}$ varies in $[0.1, 1]$ and $[1, 8]$, respectively. 

Fig.~\ref{FIG:average_logit_of_different_attacks} shows the average fuzziness comparison on different datasets when using the fuzziness-tuned method on different surrogate models. Fig.~\ref{FIG:Ablation_of_K_T_on_CIFAR10_with_VGG16} and Fig.~\ref{FIG:Ablation_of_K_T_on_CIFAR100_with_VGG16} show the average ASRs of attacks with considering VGG16 as the surrogate model on CIFAR10 and CIFAR100, and Fig.~\ref{FIG:Ablation_of_K_T_on_ImageNet_with_ResNet50} shows the average ASRs of attacks with considering ResNet50 as the surrogate model on ImageNet, when various parameter $\mathcal{K}$ and $\mathcal{T}$ is applied in our fuzziness-tuned method. In addition to Fig.~\ref{FIG:Ablation_of_K_T_on_CIFAR10_with_VGG16}, \ref{FIG:Ablation_of_K_T_on_CIFAR100_with_VGG16} and \ref{FIG:Ablation_of_K_T_on_ImageNet_with_ResNet50}, the supplementary compelements the average ASR heat maps of attacks with other surrogate model on three datasets.

As shown in subfigure (1) of Fig.~\ref{FIG:average_logit_of_different_attacks} and Fig.~\ref{FIG:Ablation_of_K_T_on_CIFAR10_with_VGG16}, in comparison with attack applied with default CE loss, with the increase of  the parameter $\mathcal{K}$ (i.e., $\mathcal{K}>1.0$),  the generated adversarial examples in most of attacks applied with our fuzziness-tuned method can achieve less average fuzziness and higher the attack success rates on CIFAR10, i.e., the transferability of generated adversarial examples can be enhanced. The results can effectively verify the Proposition~\ref{proposition:CCE_property}  in Section~\ref{sec:analysis_of_the_confidence_scaling_mechanism}.

As shown in subfigures (2)-(3) of Fig.~\ref{FIG:average_logit_of_different_attacks} and Fig.~\ref{FIG:Ablation_of_K_T_on_CIFAR100_with_VGG16} and Fig.~\ref{FIG:Ablation_of_K_T_on_ImageNet_with_ResNet50}, with the increase of parameter $\mathcal{T}$ (i.e., $\mathcal{T}>1.0$) in our fuzziness-tuned method, attacks applied with our fuzziness-tuned method can achieve less average fuzziness and greater ASRs than attacks applied with defaut CE loss on CIFAR100 and ImageNet, which demonstrates that our fuzziness-tuned method can effectively enhance the transferability of generated adversarial examples. The results can effectively verify the Proposition~\ref{proposition:TCE_property1} and Proposition~\ref{proposition:TCE_property2} in Section~\ref{sec:analysis_of_the_temperature_scaling_mechanism}.

In addition, as shown in Fig.~\ref{FIG:Ablation_of_K_T_on_CIFAR100_with_VGG16} and Fig.~\ref{FIG:Ablation_of_K_T_on_ImageNet_with_ResNet50}, our fuzziness-tuned method can achieve positive effectiveness on attacks that use the momentum and variance tuning to stabilize the update direction, e.g., MI/SINI/VMI-FGSM and SGM.  For these attacks, due to the temperature scaling mechanism can further stabilize the update direction according to Proposition~\ref{proposition:TCE_property2}, the average ASRs of MI/SINI/VMI-FGSM and SGM attacks with our fuzziness-tuned method can be increased as the parameter $\mathcal{T}$ increases. The results show that,  with the increase of parameter $\mathcal{T}$ in our fuzziness-tuned method, the update direction of the momentum or variance tuning based attacks, e.g. MI/SINI/VMI-FGSM, can be further stabilized, thereby enhancing the transferability of adversarial examples generated by these attacks. 

\begin{table*}[h]
	\caption{The untarget attack success rates (\%) on six models without defenses for CIFAR10 using various transfer attacks with the attack strength $\epsilon=8/255$. The adversarial examples are generated by VGG16 and ResNet50, respectively. $*$ denotes the attack success rates under white-box attacks. "\ding{52}" and "\ding{56}" denote whether to use fuzziness-tuned method. \textbf{Average} means to calculate the average value except $*$.}
	\label{tab:cifar10_no_defenses}
	\renewcommand\arraystretch{1.1}
	\begin{center}
		\begin{tabular}{cccccccccc}
			\hline
			Model                     & Attack                     & Fuzziness-tuned  & VGG16          & VGG19          & ResNet50           & WRN-16-4       & DenseNet121       & MobileNet-v2         & \textbf{Average} \\ \hline
			\multirow{6}{*}{VGG16}    & \multirow{2}{*}{SINI-FGSM} & \ding{56}    & 93.5*          & 70.55          & 66.3            & 76.55          & 71.05          & 74.55          & 71.8             \\
			&                            & \ding{52}   & \textbf{98.4*} & \textbf{85.3}  & \textbf{79.3}   & \textbf{86.9}  & \textbf{84.55} & \textbf{86.4}  & \textbf{84.49}   \\ \cline{2-10} 
			& \multirow{2}{*}{VMI-FGSM}  & \ding{56}    & 94.15*         & 79.1           & 74.95           & 80.0           & 76.7           & 78.15          & 77.78            \\
			&                            & \ding{52}   & \textbf{99.5*} & \textbf{89.6}  & \textbf{85.3}   & \textbf{89.9}  & \textbf{88.25} & \textbf{89.05} & \textbf{88.42}   \\ \cline{2-10} 
			& \multirow{2}{*}{FIA}       & \ding{56} & 99.5*          & 86.5           & 82.55           & 90.25          & 85.4           & 88.45          & 86.63            \\
			&                            & \ding{52}   & \textbf{99.6*} & \textbf{87.25} & \textbf{82.95}  & \textbf{90.55} & \textbf{86.3}  & \textbf{88.5}  & \textbf{87.11}   \\ \hline
			\multirow{6}{*}{ResNet50} & \multirow{2}{*}{SINI-FGSM} & \ding{56}    & 71.95          & 74.1           & 97.6*           & 91.85          & 89.2           & 85.8           & 82.58            \\
			&                            & \ding{52}   & \textbf{81.45} & \textbf{84.3}  & \textbf{99.6*}  & \textbf{96.95} & \textbf{95.1}  & \textbf{92.75} & \textbf{90.11}   \\ \cline{2-10} 
			& \multirow{2}{*}{VMI-FGSM}  & \ding{56}    & 79.15          & 80.55          & 99.0*           & 94.0           & 90.25          & 87.95          & 86.38            \\
			&                            & \ding{52}   & \textbf{83.15} & \textbf{85.15} & \textbf{100*}   & \textbf{95.7}  & \textbf{92.3}  & \textbf{89.85} & \textbf{89.23}   \\ \cline{2-10} 
			& \multirow{2}{*}{FIA}       & \ding{56} & 84.55          & 85.0           & 99.25*          & 95.7           & 94.25          & 92.05          & 90.31            \\
			&                            & \ding{52}   & \textbf{86.3}  & \textbf{86.4}  & \textbf{99.85*} & \textbf{97.0}  & \textbf{94.95} & \textbf{93.6}  & \textbf{91.65}   \\ \hline
		\end{tabular}
	\end{center}
\end{table*}
\begin{table*}[h]
	\caption{The untarget attack success rates (\%) on six models without defenses for CIFAR100 using various transfer attacks with the attack strength $\epsilon=8/255$. The adversarial examples are generated by VGG16 and ResNet50, respectively. $*$ denotes the attack success rates under white-box attacks. "\ding{52}" and "\ding{56}" denote whether to use fuzziness-tuned method. \textbf{Average} means to calculate the average value except $*$.}
	\label{tab:cifar100_no_defenses}
	\renewcommand\arraystretch{1.1}
	\begin{center}
		\begin{tabular}{cccccccccc}
			\hline
			Model                     & Attack                     & Fuzziness-tuned  & VGG16           & ResNet50           & ResNext50         & WRN-16-4       & DenseNet121       & MobileNet-v2         & \textbf{Average} \\ \hline
			\multirow{6}{*}{VGG16}    & \multirow{2}{*}{SINI-FGSM} & \ding{56}    & 96.05*          & 58.65           & 67.1           & 69.45          & 65.2           & 65.55          & 65.19            \\
			&                            & \ding{52}   & \textbf{99.5*}  & \textbf{68.8}   & \textbf{74.45} & \textbf{78.0}  & \textbf{74.3}  & \textbf{74.4}  & \textbf{73.99}   \\ \cline{2-10} 
			& \multirow{2}{*}{VMI-FGSM}  & \ding{56}    & \textbf{99.3*}  & 74.7            & 76.55          & 80.45          & 76.1           & 75.35          & 76.63            \\
			&                            & \ding{52}   & 99.2*           & \textbf{78.0}   & \textbf{80.75} & \textbf{83.7}  & \textbf{80.1}  & \textbf{80.3}  & \textbf{80.57}   \\ \cline{2-10} 
			& \multirow{2}{*}{FIA}       & \ding{56} & 94.45*          & \textbf{60.35}  & \textbf{69.25} & 71.5           & 64.1           & \textbf{69.3}  & 66.9             \\
			&                            & \ding{52}   & \textbf{94.55*} & 60.05           & 68.15          & \textbf{72.6}  & \textbf{64.9}  & 69.25          & \textbf{66.99}   \\ \hline
			\multirow{8}{*}{ResNet50} & \multirow{2}{*}{SINI-FGSM} & \ding{56}    & 68.25           & 95.6*           & 78.3           & 81.45          & 79.8           & 67.95          & 75.15            \\
			&                            & \ding{52}   & \textbf{82.45}  & \textbf{99.3*}  & \textbf{89.55} & \textbf{91.45} & \textbf{89.85} & \textbf{81.95} & \textbf{87.05}   \\ \cline{2-10} 
			& \multirow{2}{*}{VMI-FGSM}  & \ding{56}    & 80.9            & \textbf{97.55*} & 85.65          & 88.1           & 86.6           & 78.55          & 83.96            \\
			&                            & \ding{52}   & \textbf{83.3}   & 97.1*           & \textbf{88.45} & \textbf{90.5}  & \textbf{89.1}  & \textbf{81.25} & \textbf{86.52}   \\ \cline{2-10} 
			& \multirow{2}{*}{FIA}       & \ding{56} & \textbf{76.65}  & \textbf{97.9*}  & \textbf{83.9}  & \textbf{86.65} & \textbf{84.35} & 77.6           & 81.83            \\
			&                            & \ding{52}   & 76.5            & 97.6*           & \textbf{83.9}  & 86.45          & 84.25          & \textbf{78.2}  & \textbf{81.86}   \\ \cline{2-10} 
			& \multirow{2}{*}{SGM}       & \ding{56}    & 72.15           & \textbf{96.15*} & 79.15          & 82.35          & 77.15          & 75.1           & 77.18            \\
			&                            & \ding{52}   & \textbf{72.55}  & 95.55*          & \textbf{80.4}  & \textbf{83.1}  & \textbf{78.95} & \textbf{78.65} & \textbf{78.73}   \\ \hline
		\end{tabular}
	\end{center}
\end{table*}
\begin{table*}[t]
	\caption{The untarget attack success rates (\%) on six models without defenses for ImageNet using various transfer attacks with the attack strength $\epsilon=8/255$. The adversarial examples are generated by VGG16, ResNet50 and ResNet152, respectively. $*$ denotes the attack success rates under white-box attacks. "\ding{52}" and "\ding{56}" denote whether to use fuzziness-tuned method. \textbf{Average} means to calculate the average value except $*$.}
	\label{tab:imagenet_no_defenses}
	\renewcommand\arraystretch{1.1}
	\begin{center}
		\begin{tabular}{cccccccccc}
			\hline
			Model                      & Attack                     & Fuzziness-tuned  & VGG16          & VGG19         & ResNet50          & ResNet152         & Inception-v3        & Mobilenet-v2        & \textbf{Average} \\ \hline
			\multirow{6}{*}{VGG16}     & \multirow{2}{*}{SINI-FGSM} & \ding{56}    & \textbf{100*}  & 97.6          & 51.4           & 35.6           & \textbf{41.9} & 67.9          & 58.88            \\
			&                            & \ding{52}   & \textbf{100*}  & \textbf{97.8} & \textbf{54.8}  & \textbf{38.5}  & 41.2          & \textbf{70.3} & \textbf{60.52}   \\ \cline{2-10} 
			& \multirow{2}{*}{VMI-FGSM}  & \ding{56}    & 99.4*          & 97.1          & 56.9           & 42.2           & 40.7          & 69.3          & 61.24            \\
			&                            & \ding{52}   & \textbf{100*}  & \textbf{97.4} & \textbf{61.0}  & \textbf{45.7}  & \textbf{44.0} & \textbf{72.1} & \textbf{64.04}   \\ \cline{2-10} 
			& \multirow{2}{*}{FIA}       & \ding{56} & \textbf{99.8*} & 96.5          & 63.7           & 46.1           & 44.7          & 76.3          & 65.46            \\
			&                            & \ding{52}   & \textbf{99.8*} & \textbf{97.1} & \textbf{69.0}  & \textbf{49.5}  & \textbf{47.5} & \textbf{80.9} & \textbf{68.8}    \\ \hline
			\multirow{8}{*}{ResNet50}  & \multirow{2}{*}{SINI-FGSM} & \ding{56}    & 62.0           & 61.1          & 100*           & 76.6           & 47.5          & 68.4          & 63.12            \\
			&                            & \ding{52}   & \textbf{68.2}  & \textbf{66.5} & \textbf{100*}  & \textbf{80.1}  & \textbf{50.0} & \textbf{70.8} & \textbf{67.12}   \\ \cline{2-10} 
			& \multirow{2}{*}{VMI-FGSM}  & \ding{56}    & 65.1           & 63.6          & 99.9*          & 79.7           & 50.2          & 67.3          & 65.18            \\
			&                            & \ding{52}   & \textbf{70.9}  & \textbf{68.3} & \textbf{100*}  & \textbf{80.7}  & \textbf{52.5} & \textbf{71.9} & \textbf{68.86}   \\ \cline{2-10} 
			& \multirow{2}{*}{FIA}       & \ding{56} & 59.4           & 55.9          & \textbf{99.9*} & 67.5           & 42.2          & 60.4          & 57.08            \\
			&                            & \ding{52}   & \textbf{68.9}  & \textbf{65.1} & 99.8*          & \textbf{79.3}  & \textbf{49.7} & \textbf{69.9} & \textbf{66.58}   \\ \cline{2-10} 
			& \multirow{2}{*}{SGM}       & \ding{56}    & 74.4           & 68.0          & 99.7*          & \textbf{74.6}  & 49.0          & 72.7          & 67.74            \\
			&                            & \ding{52}   & \textbf{78.5}  & \textbf{74.2} & \textbf{100*}  & 73.8           & \textbf{50.6} & \textbf{77.3} & \textbf{70.88}   \\ \hline
			\multirow{8}{*}{ResNet152} & \multirow{2}{*}{SINI-FGSM} & \ding{56}    & 55.1           & 54.9          & \textbf{82.7}  & 99.9*          & 46.9          & 62.0          & 60.32            \\
			&                            & \ding{52}   & \textbf{60.3}  & \textbf{60.1} & 82.6           & \textbf{100*}  & \textbf{51.2} & \textbf{68.1} & \textbf{64.46}   \\ \cline{2-10} 
			& \multirow{2}{*}{VMI-FGSM}  & \ding{56}    & 58.1           & 58.0          & 83.5           & 99.7*          & 50.5          & 62.5          & 62.52            \\
			&                            & \ding{52}   & \textbf{62.0}  & \textbf{58.8} & \textbf{84.5}  & \textbf{100*}  & \textbf{52.2} & \textbf{66.5} & \textbf{64.8}    \\ \cline{2-10} 
			& \multirow{2}{*}{FIA}       & \ding{56} & 48.2           & 46.5          & 69.6           & \textbf{99.8*} & 40.2          & 54.3          & 51.76            \\
			&                            & \ding{52}   & \textbf{56.8}  & \textbf{54.9} & \textbf{80.3}  & 99.4*          & \textbf{47.9} & \textbf{62.5} & \textbf{60.48}   \\ \cline{2-10} 
			& \multirow{2}{*}{SGM}       & \ding{56}    & 67.8           & 63.4          & 83.2           & 99.5*          & 48.7          & 72.6          & 67.14            \\
			&                            & \ding{52}   & \textbf{74.7}  & \textbf{68.8} & \textbf{84.2}  & \textbf{99.9*} & \textbf{51.2} & \textbf{76.5} & \textbf{71.08}   \\ \hline
		\end{tabular}
	\end{center}
\end{table*}
\begin{table*}[t]
	\caption{The untarget attack success rates (\%) on eight advanced defense methods for ImageNet using various transfer attacks with the attack strength $\epsilon=8/255$. The adversarial examples are generated by VGG16, ResNet50 and ResNet152, respectively. $*$ denotes the attack success rates under white-box attacks. "\ding{52}" and "\ding{56}" denote whether to use fuzziness-tuned method. \textbf{Average} means to calculate the average value except $*$.}
	\label{tab:imagenet_with_defenses}
	\renewcommand\arraystretch{1.1}
	\begin{center}
		\begin{tabular}{cccccccccccc}
			\hline
			Model                      & Attack                     & Fuzziness-tuned  & RP            & Bit-Red       & JPEG          & FD            & NRP           & RS            & Inc-v3$_{adv}$    & IncRes-v2$^{ens}_{adv}$ & \textbf{Average} \\ \hline
			\multirow{6}{*}{VGG16}     & \multirow{2}{*}{SINI-FGSM} & \ding{56}    & 61.6          & 58.7          & 54.2          & 57.0          & 27.7          & \textbf{45.0} & 23.2          & \textbf{12.9}     & 42.54            \\
			&                            & \ding{52}   & \textbf{63.9} & \textbf{62.3} & \textbf{55.8} & \textbf{59.0} & \textbf{29.3} & 43.5          & \textbf{23.4} & 12.1              & \textbf{43.66}   \\ \cline{2-12} 
			& \multirow{2}{*}{VMI-FGSM}  & \ding{56}    & 64.0          & 62.4          & 55.9          & 60.9          & 30.7          & 46.5          & 22.3          & 11.7              & 44.3             \\
			&                            & \ding{52}   & \textbf{66.4} & \textbf{64.4} & \textbf{60.8} & \textbf{63.4} & \textbf{31.8} & \textbf{49.0} & \textbf{23.5} & \textbf{13.0}     & \textbf{46.54}   \\ \cline{2-12} 
			& \multirow{2}{*}{FIA}       & \ding{56} & 70.0          & 69.6          & 56.4          & 61.9          & \textbf{39.5} & 56.0          & 24.1          & 11.2              & 48.59            \\
			&                            & \ding{52}   & \textbf{73.2} & \textbf{74.8} & \textbf{62.9} & \textbf{67.9} & 32.1          & \textbf{58.5} & \textbf{24.8} & \textbf{11.5}     & \textbf{50.71}   \\ \hline
			\multirow{8}{*}{ResNet50}  & \multirow{2}{*}{SINI-FGSM} & \ding{56}    & 62.7          & 61.0          & 55.3          & 57.0          & \textbf{31.9} & 47.5          & \textbf{27.5} & 14.6              & 44.69            \\
			&                            & \ding{52}   & \textbf{62.9} & \textbf{61.6} & \textbf{58.6} & \textbf{61.2} & 31.5          & \textbf{49.0} & 27.1          & \textbf{15.1}     & \textbf{45.88}   \\ \cline{2-12} 
			& \multirow{2}{*}{VMI-FGSM}  & \ding{56}    & 61.4          & 59.3          & 57.5          & 60.2          & 33.6          & 46.5          & 24.8          & 15.5              & 44.85            \\
			&                            & \ding{52}   & \textbf{64.9} & \textbf{62.1} & \textbf{61.4} & \textbf{62.0} & \textbf{36.1} & \textbf{52.0}          & \textbf{28.2}          & \textbf{16.6}     & \textbf{47.91}   \\ \cline{2-12} 
			& \multirow{2}{*}{FIA}       & \ding{56} & 54.8          & 52.4          & 48.2          & 50.3          & \textbf{39.0} & 45.0          & 24.3          & 11.9              & 40.74            \\
			&                            & \ding{52}   & \textbf{64.2} & \textbf{62.2} & \textbf{57.5} & \textbf{60.8} & 37.2          & \textbf{53.0} & \textbf{26.1} & \textbf{13.8}     & \textbf{46.85}   \\ \cline{2-12} 
			& \multirow{2}{*}{SGM}       & \ding{56}    & 70.3          & 64.7          & 58.7          & 64.4          & 38.0          & 48.5          & 25.2          & \textbf{12.7}     & 47.81            \\
			&                            & \ding{52}   & \textbf{73.1} & \textbf{67.6} & \textbf{60.3} & \textbf{65.4} & \textbf{38.7} & \textbf{50.5} & \textbf{25.6} & 12.1              & \textbf{49.16}   \\ \hline
			\multirow{8}{*}{ResNet152} & \multirow{2}{*}{SINI-FGSM} & \ding{56}    & 59.2          & 55.8          & 53.3          & 54.7          & 27.4          & 47.0          & \textbf{27.9} & \textbf{17.0}     & 42.79            \\
			&                            & \ding{52}   & \textbf{62.6} & \textbf{58.4} & \textbf{55.3} & \textbf{57.2} & \textbf{30.3} & \textbf{48.0} & 27.6          & 16.4              & \textbf{44.48}   \\ \cline{2-12} 
			& \multirow{2}{*}{VMI-FGSM}  & \ding{56}    & 59.3          & 55.4          & 54.1          & 55.7          & 28.8          & 42.5          & 24.9          & 15.8              & 42.06            \\
			&                            & \ding{52}   & \textbf{61.2} & \textbf{58.0} & \textbf{56.8} & \textbf{57.1} & \textbf{31.2} & \textbf{46.0} & \textbf{27.4} & \textbf{19.3}     & \textbf{44.63}   \\ \cline{2-12} 
			& \multirow{2}{*}{FIA}       & \ding{56} & 50.3          & 46.3          & 44.9          & 47.9          & \textbf{36.7} & 41.5          & 23.9          & 12.5              & 38.0             \\
			&                            & \ding{52}   & \textbf{58.2} & \textbf{54.8} & \textbf{53.7} & \textbf{57.1} & 33.8          & \textbf{47.0} & \textbf{25.8} & \textbf{15.9}     & \textbf{43.29}   \\ \cline{2-12} 
			& \multirow{2}{*}{SGM}       & \ding{56}    & 66.0          & 62.4          & 56.5          & 61.3          & 34.1          & \textbf{47.0} & 25.4          & \textbf{13.8}     & 45.81            \\
			&                            & \ding{52}   & \textbf{71.9} & \textbf{65.9} & \textbf{59.3} & \textbf{64.6} & \textbf{34.5} & \textbf{47.0} & \textbf{26.3} & 13.6              & \textbf{47.89}   \\ \hline
		\end{tabular}
	\end{center}
\end{table*}
\begin{figure*}[ht]
	\begin{center}
		\centerline{\includegraphics[width=\textwidth]{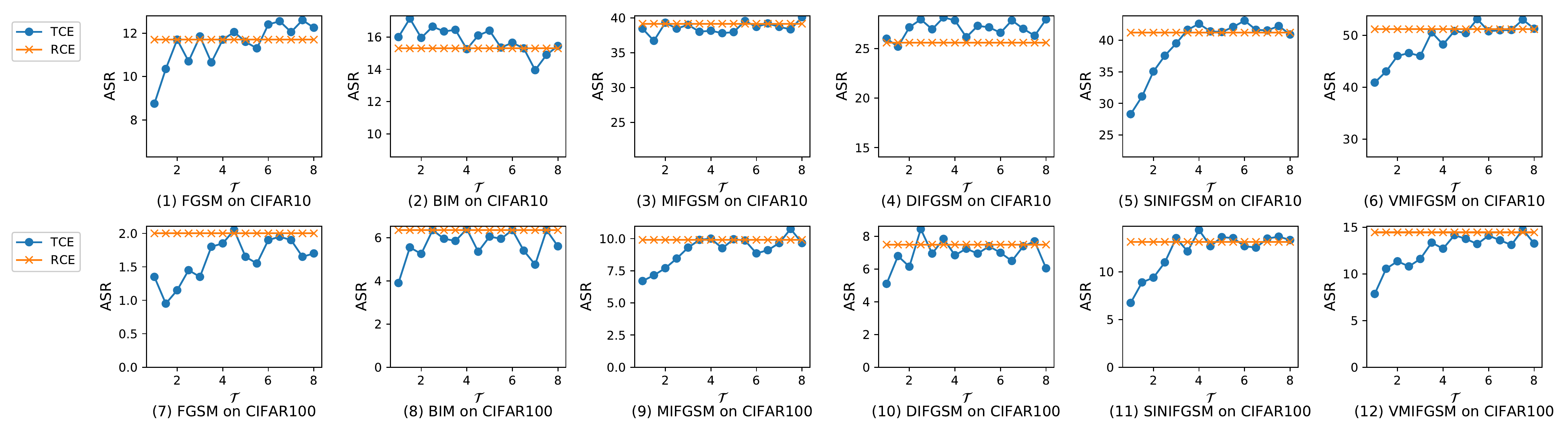}}
		\caption{The target attack success rates (\%) from ResNet50 as the surrogate model to VGG16 as the victim model on CIFAR10/100.}\vspace{-5mm}
		\label{FIG:Relationship_between_TCE_and_RCE}
	\end{center}
\end{figure*}
\begin{table*}[t]
	\caption{The average untarget attack success rates (\%) on five models without defenses for CIFAR10/100 and ImageNet using various transfer attacks with the attack strength $\epsilon=8/255$. The adversarial examples are generated by VGG16, ResNet50 and ResNet152, respectively.}
	\label{tab:comparisons_between_fce_and_rce}
	\renewcommand\arraystretch{1.1}
	\begin{center}
		\begin{tabular}{ccccccccc}
			\hline
			Dataset                   & Model                      & Method & FGSM           & I-FGSM         & MI-FGSM        & DI-FGSM        & SINI-FGSM      & VMI-FGSM       \\ \hline
			\multirow{4}{*}{CIFAR10}  & \multirow{2}{*}{VGG16}     & RCE  & 37.54          & 66.47          & 69.81          & 64.05          & 84.02          & 75.78          \\
			&                            & Fuzziness-tuned  & \textbf{58.23} & \textbf{76.94} & \textbf{86.85} & \textbf{81.76} & \textbf{84.49} & \textbf{88.42} \\ \cline{2-9} 
			& \multirow{2}{*}{ResNet50}  & RCE  & 41.39          & 73.64          & 78.64          & 71.84          & 89.5           & 83.78          \\
			&                            & Fuzziness-tuned  & \textbf{55.01} & \textbf{80.5}  & \textbf{88.8}  & \textbf{82.92} & \textbf{90.11} & \textbf{89.23} \\ \hline
			\multirow{4}{*}{CIFAR100} & \multirow{2}{*}{VGG16}     & RCE  & 59.94          & 50.14          & 68.8           & 57.84          & 72.7           & 77.4           \\
			&                            & Fuzziness-tuned  & \textbf{62.18} & \textbf{53.86} & \textbf{72.37} & \textbf{62.05} & \textbf{73.99} & \textbf{80.57} \\ \cline{2-9} 
			& \multirow{2}{*}{ResNet50}  & RCE  & 62.21          & 74.66          & 80.65          & 75.8           & 86.55          & 84.86          \\
			&                            & Fuzziness-tuned  & \textbf{64.97} & \textbf{75.08} & \textbf{81.9}  & \textbf{76.91} & \textbf{87.05} & \textbf{86.52} \\ \hline
			\multirow{6}{*}{ImageNet} & \multirow{2}{*}{VGG16}     & RCE  & 30.06          & 26.36          & 45.24          & 37.46          & 57.58          & 57.14          \\
			&                            & Fuzziness-tuned  & \textbf{35.08} & \textbf{33.4}  & \textbf{50.02} & \textbf{44.04} & \textbf{60.52} & \textbf{64.04} \\ \cline{2-9} 
			& \multirow{2}{*}{ResNet50}  & RCE  & 26.26          & 20.98          & 48.28          & 41.92          & 65.02          & 63.78          \\
			&                            & Fuzziness-tuned  & \textbf{32.08} & \textbf{25.34} & \textbf{50.22} & \textbf{46.7}  & \textbf{67.12} & \textbf{68.86} \\ \cline{2-9} 
			& \multirow{2}{*}{ResNet152} & RCE  & 26.38          & 22.38          & 46.12          & 38.6           & 63.66          & 60.5           \\
			&                            & Fuzziness-tuned  & \textbf{30.78} & \textbf{25.18} & \textbf{48.14} & \textbf{44.68} & \textbf{64.46} & \textbf{64.8}  \\ \hline
		\end{tabular}
	\end{center}
\end{table*}

\subsection{Effect of Fuzziness-tuned Method on Transferability of Generated Adversarial Examples}
\label{sec:effects_of_FCE_on_transferability}
In this section, the effectiveness of our fuzziness-tuned method on attack success rates (ASRs) is evaluated. Particularly, in the attacks integrated with our fuzziness-tuned method, the $L_{FCE}$ defined in Equation~(\ref{equation:fece}) is used as loss function to generate adversarial examples. Particularly, for FGSM, I/DI/MI/SINI/VMI-FGSM and SGM, the default loss function is the cross-entropy (CE) loss. 
Additionally, the logit output with respect to the ground truth label is used to calculate the aggregate gradient in FIA (as shown in Equations~(\ref{equation:fia_weight}) and (\ref{equation:fia_avg_weight})), which will lose the logit output information of other wrong labels. To utilize the lost information,  our $L_{FCE}$ loss is explored to calculate more effective aggregate gradient. For our fuzziness-tuned method, the parameters $\mathcal{T}$ and $\mathcal{K}$ are selected as the best combination among different attacks and datasets, based on the average ASR heat maps shown in Fig.~\ref{FIG:Ablation_of_K_T_on_CIFAR10_with_VGG16}, Fig.~\ref{FIG:Ablation_of_K_T_on_CIFAR100_with_VGG16} and Fig.~\ref{FIG:Ablation_of_K_T_on_ImageNet_with_ResNet50}.

\subsubsection{Attack against the Victim Models without Defenses}
\label{sec:attacking_models_without_defenses} 

Tables~\ref{tab:cifar10_no_defenses}, \ref{tab:cifar100_no_defenses} and \ref{tab:imagenet_no_defenses} show the average ASRs on the victim models in SINI/VMI-FGSM and FIA w/ or w/o our fuzziness-tuned method on CIFAR10/100 and ImageNet datasets, respectively. 

As shown in Table~\ref{tab:cifar10_no_defenses}, through integrating with our fuzziness-tuned mehod, the average ASRs of these attacks are increased by 2.85\% to 12.69\% for SINI/VMI-FGSM and 0.48\% to 1.34\% for FIA on the victim models. When compared with $L_{logit}$ loss, the aggregate gradient computed by our $L_{FCE}$ loss can lead to higher transferability of FIA. That means, our fuzziness-tuned method can improve the ASR on the surrogate model and enhance the transferability of generated adversarial examples on the victim models. 


As shown in Table~\ref{tab:cifar100_no_defenses}, with our fuzziness-tuned method, the highest average ASR of all baselines can be increased from 83.96\% to 87.05\%. In particular, the average ASR is increased by 2.56\% to 11.9\% for SINI/VMI-FGSM and 1.55\% for SGM under black-box setting. Additionally, the ASR of our fuzziness-tuned FIA setting a large parameter $\mathcal{T}$ approach to that of the logit output-based FIA.



As shown in Table~\ref{tab:imagenet_no_defenses}, the average ASR is increased by 1.64\% to 4.14\% for SINI/VMI-FGSM and 3.14\% to 3.94\% for SGM under black-box setting. Additionally, in comparison with FIA that uses the logit output to compute the aggregate gradient, our fuzziness-tuned method can increase the average ASR by 3.34\% to 8.72\% because of considering more information of the logit outputs corresponding to other wrong labels in the aggregate gradient computation. The results show that, with our fuzziness-tuned method, the highest ASRs in all baselines can be improved from 67.74\% to 71.08\%, representing the transferability of generated adversarial examples can be enhanced.

Tables~\ref{tab:cifar10_no_defenses}, \ref{tab:cifar100_no_defenses} and \ref{tab:imagenet_no_defenses} show that our fuzziness-tuned method can improve the ASRs on the victim models on basis of keeping or increasing the ASR on the surrogate model for SINI/VMI-FGSM and SGM. 

\subsubsection{Attack against the Victim Models with Defenses}
\label{sec:attacking_models_with_defenses}
The existing efforts have shown that the current advanced defense methods are vulnerable to SINI/VMI-FGSM~\cite{SI-NI-FGSM, VMI-FGSM}. Hence, in this section, Table~\ref{tab:imagenet_with_defenses} shows that the effectiveness of our fuzziness-tuned method is only evaluated on the latest attacks, e.g.,  SINI/VMI-FGSM, FIA and SGM, to against the victim models with the advanced defenses on ImageNet. Particularly, eight advanced defense methods, including six input modifications and two adversarial trained models, are considered in the evaluations. 
As shown in Table~\ref{tab:imagenet_with_defenses}, our fuzziness-tuned method can improve the average ASR by 1.12\% to 3.06\% for SINI/VMI-FGSM, 2.12\% to 6.11\% for FIA and 1.35\% to 2.08\% for SGM, to against the victim models with eight advanced defense methods, respectively. Specifically, the ASRs of VMI-FGSM and FIA with our fuzziness-tuned method can be increased on the victim models with all defense methods. The average ASRs of SINI-FGSM and SGM with our fuzziness-tuned method can be increased as well on the victim models with six input modification methods and maintained on the victim models that are adversarially trained. Hence, the results can demonstrate that our fuzziness-tuned method can improve the ASRs on the victim models with defense methods. i.e., enhancing the transferability of the latest attacks against the victim models that are integrated with the advanced defense methods.


\subsection{Comparison between $L_{RCE}$ and our $L_{FCE}$}
\label{sec:comparison_between_our_FCE_and_RCE}
The RCE loss~\cite{RCE} was proposed to significantly improve the transferability of the current transfer-based attacks under the target attack setting, which gave a geometric interpretation of the logit gradient that the RCE loss guides the logit to be updated in the direction of implicitly maximizing its rank distance from the ground truth label. 

Recall that, Proposition~\ref{proposition:relationship_between_TCE_and_RCE} theoretically proves that the gradient sign of our $L_{FCE}$ loss with respect to the input $x$ is equivalent to that of the  $L_{RCE}$ loss when the parameter $\mathcal{T}$ goes to $+\infty$. Proposition~\ref{proposition:TCE_property2} demonstrates that our  $L_{FCE}$ loss can stabilize the update direction of generating adversarial example when the parameter $\mathcal{T}$ goes to $+\infty$. Therefore, in addition to the intuitive geometric interpretation, the great performance of the RCE loss~\cite{RCE} can also be explained as stabilizing the update direction to indirectly reduce the fuzziness.

To verify Propositions~\ref{proposition:TCE_property2} and \ref{proposition:relationship_between_TCE_and_RCE} empirically, the comparison experiments between the  $L_{RCE}$ and our  $L_{FCE}$ are conducted on CIFAR10/100 for FGSM and I/MI/DI/SINI/VMI-FGSM under target attack setting and untarget attack setting, respectively, as shown in Fig.~\ref{FIG:Relationship_between_TCE_and_RCE} and Table~\ref{tab:comparisons_between_fce_and_rce}. 

Under the target attack setting, ResNet50 and VGG16 are considered as the surrogate model and victim model respectively and the target label is selected randomly. As shown in Fig.~\ref{FIG:Relationship_between_TCE_and_RCE}, with increase of the parameter $\mathcal{T}$, the target ASR of attacks with our $L_{FCE}$ loss is increased gradually and approaches to that with $L_{RCE}$ loss, which experimentally verifies Proposition~\ref{proposition:TCE_property2} and Proposition~\ref{proposition:relationship_between_TCE_and_RCE}, respectively.



Under the untarget attack setting, Table~\ref{tab:comparisons_between_fce_and_rce} shows that the ASRs of attacks with our  fuzziness-tuned method (i.e., $L_{FCE}$ loss) are significantly better than that with the RCE loss (i.e., $L_{RCE}$) on various datasets and surrogate models. Due to the RCE loss is a fixed function and does not included any hyper-parameter to adjust the stability of update direction,  the adversarial example generated by the RCE loss (i.e., $L_{RCE}$) might get stuck in the overfitting fuzzy domain, resulting in the poor transferability. The $L_{FCE}$ loss in our fuzziness-tuned method can tune the fuzziness of generated adversarial examples to ensure the generated adversarial examples skip out of both overfitting and underfitting fuzzy domain, thereby improving the ASRs of generated adversarial examples on the victim models, i.e., enhancing the transferability of generated adversarial examples. 

\section{Related Works}
\label{sec:related_work}
\subsection{Attacks}
Adversarial attacks are categorized into white-box attacks and black-box attacks. The former has known the architecture and parameters of the victim model, such as the projected gradient descent (PGD)~\cite{Madry-AT}, Carlini \& Wagner~\cite{CW} and adaptive PGD~\cite{APGD}, etc. The latter does not have any information about the victim model except for its output, which can be divided into query-based attacks and transfer-based attacks. The query-based attacks can catch the output of the victim model, including have zeroth order optimization~\cite{ZOO} and Square~\cite{Square}, etc. The transfer-based attacks uses the surrogate model to estimate the perturbation. In this paper, we concentrate on investigating the transferability of transfer-based attacks, i.e., the transferability of adversarial examples generated by transfer-based attacks from the surrogate model to the victim model.

Generally, the transfer-based attacks can be divided as the family of gradient-based attacks, input transformations and other types of attacks. In particular, the family of gradient-based attacks, including \textbf{FGSM}~\cite{FGSM},  \textbf{I-FGSM}~\cite{I-FGSM}, \textbf{MI-FGSM}~\cite{MI-FGSM}, \textbf{NI-FGSM}~\cite{SI-NI-FGSM}, \textbf{VMI-FGSM}~\cite{VMI-FGSM}, have been mentioned in Section~\ref{Family_of_FGSM}.

Input transformations, including diverse input method~\cite{DI-FGSM}, translation-invariant method~\cite{TI-FGSM}, scale-invariant method~\cite{SI-NI-FGSM}, utilize the invariant properties of deep neural networks to expand input diversity, thereby avoiding the generated adversarial examples falling into the poor local optimum. Specifically, diverse input method~\cite{DI-FGSM} can randomly resize the input images and randomly pad zeros around the input images. Translation-invariant method~\cite{TI-FGSM} optimized a perturbation over an ensemble of translated images through convolving the gradient at the untranslated image with a pre-defined kernel to improve the efficiency of attacks. Scale-Invariant method~\cite{SI-NI-FGSM} can optimize the adversarial perturbations over the scale copies of the input images through involving the discovered scale-invariant property of deep learning models.

The feature importance-aware attack (\textbf{FIA})~\cite{FIA} can be considered as the transfer-based attack as well, which can corrupt the middle layer features by adding weights to these features for avoiding overfitting. The neuron attribution-based attack has been proposed as well to solve the inaccurate neuron importance estimations of the existing feature-level attacks, such as \textbf{FIA}~\cite{FIA}.  In addition, the skip gradient method (\textbf{SGM})~\cite{SGM} and the relative cross-entropy (RCE loss)~\cite{RCE} have been also proposed to ensure the generated adversarial examples skip out of poor local maxima and improve the transferability of the generated adversarial examples, which are mentioned in Section~\ref{Family_of_FGSM} as well.

However, when a low attack stength (e.g.,  $\epsilon=8/255$) is applied in these transfer-based attacks, the attack success rate of generated adversarial examples on the surrogate model will be much larger than that on the victim model, which means the generated adversarial examples achieve poor transferability from the surrogate model to the victim model. Hence, in our paper a fuzziness-tuned method is proposed to be applied into these transfer-based attacks with low  attack stength to improve the attack success rate of generated adversarial examples on the victim model, i.e., enhancing the transferability of adversarial examples generated by these  transfer-based attacks with low attack stength.

\subsection{Defenses}
\label{sec:defenses}
Adversarial defense methods can be mainly divided as adversarial training~\cite{Ensemble-Adversarial-Training,Madry-AT,Free-AT,Fast-AT,Trades} and input modification~\cite{RP,Bit-Red,JPEG,FD,RS,NRP}. The former can modify the parameters of deep neural networks via training with adversarial examples. The latter can purify the input of deep learning models and eliminates the effect of the added perturbation.

Adversarial training can achieve great robustness against the white-box attacks. For example,  Madry {\em et al.}~\cite{Madry-AT} firstly proposed vanilla adversarial training (AT) with adversarial examples generated by project gradient descent (PGD). To speed up the vanilla AT, free AT~\cite{Free-AT} was proposed to eliminate the overhead cost of generating adversarial examples by recycling the gradient information computed. Due to the large cost of PGD to generate adversarial examples, Wong {\em et al.}~\cite{Fast-AT} explored a fast adversarial training method, which uses a faster attack, i.e. FGSM, to generate adversarial examples. However, these adversarial training methods degrade the performance of the trained model on the natural examples. To this end, TRADES~\cite{Trades} was proposed to balance the performance of deep learning models on the natural and adversarial examples. Additionally, Tram{\`{e}}r {\em et al.}~\cite{Ensemble-Adversarial-Training} proposed the ensemble adversarial training method, which can effectively and specifically defense adversarial examples generated by transfer-based attacks.

Input modifications can be divided as input transformation methods~\cite{RP,Bit-Red,JPEG} and input purification methods~\cite{FD,RS,NRP}. The input transformation methods can destroy the spatial characteristics of the added perturbation, thereby reducing the negative effect of adversarial examples. The input purification methods aim to restore the adversarial example to the corresponding normal example.

Although these defense methods can improve the robustness of models in DNNs and eliminate the perturbation added in adversarial examples, existing efforts have shown that these defense methods are vulnerable to the latest transfer-based attacks, such as SINI/VMI-FGSM\cite{SI-NI-FGSM,VMI-FGSM}. Hence, in our paper, the fuzziness-tuned method is proposed to be applied into these transfer-based attacks, aiming to further investigate the vulnerability of existing defense methods and generate more "real" adversarial examples to enhance the robustness of the training models on deep neural networks.


\section{Conclusion}
\label{sec:conclusion}
In this paper, we investigated the essential reason for the poor transferability of adversarial examples generated by transfer-based attacks with low attack strength and found that the issue appeared due to generated adversarial examples easily to falls into a special area, defined as the fuzzy domain in our paper. Then, a fuzziness-tuned method was proposed to tune the fuzziness of adversarial examples during their generation and ensure the generated adversarial examples can skip out of the fuzzy domain, thereby enhancing the transferability of the generated adversarial examples. Particularly, the proposed fuzziness-tuned method consists of the confidence scaling mechanism and the temperature scaling mechanism, in which the former can accelerate the decline of the fuzziness directly by increasing the gradient descent weight, while the latter can stabilize the update direction to indirectly guarantee the stability of the fuzziness decline. The extensive experiment results on CIFAR10/100 and ImageNet demonstrated our fuzziness-tuned method can effectively improve the attack success rates of transfer-based attacks on victim models and enhance the transferability of generated adversarial examples in comparison with baselines and existing methods.


\bibliographystyle{IEEEtran}
\bibliography{ref}

\begin{thebibliography}{10}
\providecommand{\url}[1]{#1}
\csname url@samestyle\endcsname
\providecommand{\newblock}{\relax}
\providecommand{\bibinfo}[2]{#2}
\providecommand{\BIBentrySTDinterwordspacing}{\spaceskip=0pt\relax}
\providecommand{\BIBentryALTinterwordstretchfactor}{4}
\providecommand{\BIBentryALTinterwordspacing}{\spaceskip=\fontdimen2\font plus
\BIBentryALTinterwordstretchfactor\fontdimen3\font minus
  \fontdimen4\font\relax}
\providecommand{\BIBforeignlanguage}[2]{{%
\expandafter\ifx\csname l@#1\endcsname\relax
\typeout{** WARNING: IEEEtran.bst: No hyphenation pattern has been}%
\typeout{** loaded for the language `#1'. Using the pattern for}%
\typeout{** the default language instead.}%
\else
\language=\csname l@#1\endcsname
\fi
#2}}
\providecommand{\BIBdecl}{\relax}
\BIBdecl

\bibitem{Linear-property}
I.~J. Goodfellow, J.~Shlens, and C.~Szegedy, ``Explaining and harnessing
  adversarial examples,'' in \emph{{ICLR} (Poster)}, 2015.

\bibitem{Adversarial-examples}
C.~Szegedy, W.~Zaremba, I.~Sutskever, J.~Bruna, D.~Erhan, I.~J. Goodfellow, and
  R.~Fergus, ``Intriguing properties of neural networks,'' in \emph{{ICLR}
  (Poster)}, 2014.

\bibitem{FGSM}
I.~J. Goodfellow, J.~Shlens, and C.~Szegedy, ``Explaining and harnessing
  adversarial examples,'' in \emph{{ICLR} (Poster)}, 2015.

\bibitem{I-FGSM}
A.~Kurakin, I.~J. Goodfellow, and S.~Bengio, ``Adversarial examples in the
  physical world,'' in \emph{{ICLR} (Workshop)}.\hskip 1em plus 0.5em minus
  0.4em\relax OpenReview.net, 2017.

\bibitem{MI-FGSM}
Y.~Dong, F.~Liao, T.~Pang, H.~Su, J.~Zhu, X.~Hu, and J.~Li, ``Boosting
  adversarial attacks with momentum,'' in \emph{{CVPR}}.\hskip 1em plus 0.5em
  minus 0.4em\relax Computer Vision Foundation / {IEEE} Computer Society, 2018,
  pp. 9185--9193.

\bibitem{DI-FGSM}
C.~Xie, Z.~Zhang, Y.~Zhou, S.~Bai, J.~Wang, Z.~Ren, and A.~L. Yuille,
  ``Improving transferability of adversarial examples with input diversity,''
  in \emph{{CVPR}}.\hskip 1em plus 0.5em minus 0.4em\relax Computer Vision
  Foundation / {IEEE}, 2019, pp. 2730--2739.

\bibitem{SI-NI-FGSM}
J.~Lin, C.~Song, K.~He, L.~Wang, and J.~E. Hopcroft, ``Nesterov accelerated
  gradient and scale invariance for adversarial attacks,'' in
  \emph{{ICLR}}.\hskip 1em plus 0.5em minus 0.4em\relax OpenReview.net, 2020.

\bibitem{VMI-FGSM}
X.~Wang and K.~He, ``Enhancing the transferability of adversarial attacks
  through variance tuning,'' in \emph{{CVPR}}.\hskip 1em plus 0.5em minus
  0.4em\relax Computer Vision Foundation / {IEEE}, 2021, pp. 1924--1933.

\bibitem{TI-FGSM}
Y.~Dong, T.~Pang, H.~Su, and J.~Zhu, ``Evading defenses to transferable
  adversarial examples by translation-invariant attacks,'' in
  \emph{{CVPR}}.\hskip 1em plus 0.5em minus 0.4em\relax Computer Vision
  Foundation / {IEEE}, 2019, pp. 4312--4321.

\bibitem{FIA}
Z.~Wang, H.~Guo, Z.~Zhang, W.~Liu, Z.~Qin, and K.~Ren, ``Feature
  importance-aware transferable adversarial attacks,'' in \emph{{ICCV}}.\hskip
  1em plus 0.5em minus 0.4em\relax {IEEE}, 2021, pp. 7619--7628.

\bibitem{NAA}
J.~Zhang, W.~Wu, J.~Huang, Y.~Huang, W.~Wang, Y.~Su, and M.~R. Lyu, ``Improving
  adversarial transferability via neuron attribution-based attacks,''
  \emph{CoRR}, vol. abs/2204.00008, 2022.

\bibitem{SGM}
D.~Wu, Y.~Wang, S.~Xia, J.~Bailey, and X.~Ma, ``Skip connections matter: On the
  transferability of adversarial examples generated with resnets,'' in
  \emph{{ICLR}}.\hskip 1em plus 0.5em minus 0.4em\relax OpenReview.net, 2020.

\bibitem{ImageNet}
O.~Russakovsky, J.~Deng, H.~Su, J.~Krause, S.~Satheesh, S.~Ma, Z.~Huang,
  A.~Karpathy, A.~Khosla, M.~S. Bernstein, A.~C. Berg, and L.~Fei{-}Fei,
  ``Imagenet large scale visual recognition challenge,'' \emph{Int. J. Comput.
  Vis.}, vol. 115, no.~3, pp. 211--252, 2015.

\bibitem{RCE}
C.~Zhang, P.~Benz, A.~Karjauv, J.~Cho, K.~Zhang, and I.~S. Kweon,
  ``Investigating top-k white-box and transferable black-box attack,'' in
  \emph{{CVPR}}.\hskip 1em plus 0.5em minus 0.4em\relax {IEEE}, 2022, pp.
  15\,064--15\,073.

\bibitem{Knowledge-distillation}
G.~E. Hinton, O.~Vinyals, and J.~Dean, ``Distilling the knowledge in a neural
  network,'' \emph{CoRR}, vol. abs/1503.02531, 2015.

\bibitem{CIFAR}
A.~Krizhevsky, ``Learning multiple layers of features from tiny images,'' 2009.

\bibitem{VGG}
K.~Simonyan and A.~Zisserman, ``Very deep convolutional networks for
  large-scale image recognition,'' in \emph{{ICLR}}, 2015.

\bibitem{ResNet}
K.~He, X.~Zhang, S.~Ren, and J.~Sun, ``Deep residual learning for image
  recognition,'' in \emph{{CVPR}}.\hskip 1em plus 0.5em minus 0.4em\relax
  {IEEE} Computer Society, 2016, pp. 770--778.

\bibitem{ResNext}
S.~Xie, R.~B. Girshick, P.~Doll{\'{a}}r, Z.~Tu, and K.~He, ``Aggregated
  residual transformations for deep neural networks,'' in \emph{{CVPR}}.\hskip
  1em plus 0.5em minus 0.4em\relax {IEEE} Computer Society, 2017, pp.
  5987--5995.

\bibitem{WideResNet}
S.~Zagoruyko and N.~Komodakis, ``Wide residual networks,'' in
  \emph{{BMVC}}.\hskip 1em plus 0.5em minus 0.4em\relax {BMVA} Press, 2016.

\bibitem{Inception-v3}
C.~Szegedy, V.~Vanhoucke, S.~Ioffe, J.~Shlens, and Z.~Wojna, ``Rethinking the
  inception architecture for computer vision,'' in \emph{{CVPR}}.\hskip 1em
  plus 0.5em minus 0.4em\relax {IEEE} Computer Society, 2016, pp. 2818--2826.

\bibitem{DenseNet}
G.~Huang, Z.~Liu, L.~van~der Maaten, and K.~Q. Weinberger, ``Densely connected
  convolutional networks,'' in \emph{{CVPR}}.\hskip 1em plus 0.5em minus
  0.4em\relax {IEEE} Computer Society, 2017, pp. 2261--2269.

\bibitem{MobileNet}
M.~Sandler, A.~G. Howard, M.~Zhu, A.~Zhmoginov, and L.~Chen, ``Mobilenetv2:
  Inverted residuals and linear bottlenecks,'' in \emph{{CVPR}}.\hskip 1em plus
  0.5em minus 0.4em\relax Computer Vision Foundation / {IEEE} Computer Society,
  2018, pp. 4510--4520.

\bibitem{pytorch-image-models}
R.~Wightman, ``Pytorch image models,''
  \url{https://github.com/rwightman/pytorch-image-models}, 2019.

\bibitem{torchvision-models}
H.~Huang, ``torchvision.models,''
  \url{https://pytorch.org/vision/stable/models.html}, 2017.

\bibitem{RP}
C.~Xie, J.~Wang, Z.~Zhang, Z.~Ren, and A.~L. Yuille, ``Mitigating adversarial
  effects through randomization,'' in \emph{{ICLR} (Poster)}.\hskip 1em plus
  0.5em minus 0.4em\relax OpenReview.net, 2018.

\bibitem{Bit-Red}
W.~Xu, D.~Evans, and Y.~Qi, ``Feature squeezing: Detecting adversarial examples
  in deep neural networks,'' in \emph{{NDSS}}.\hskip 1em plus 0.5em minus
  0.4em\relax The Internet Society, 2018.

\bibitem{JPEG}
C.~Guo, M.~Rana, M.~Ciss{\'{e}}, and L.~van~der Maaten, ``Countering
  adversarial images using input transformations,'' in \emph{{ICLR}
  (Poster)}.\hskip 1em plus 0.5em minus 0.4em\relax OpenReview.net, 2018.

\bibitem{FD}
Z.~Liu, Q.~Liu, T.~Liu, N.~Xu, X.~Lin, Y.~Wang, and W.~Wen, ``Feature
  distillation: Dnn-oriented {JPEG} compression against adversarial examples,''
  in \emph{{CVPR}}.\hskip 1em plus 0.5em minus 0.4em\relax Computer Vision
  Foundation / {IEEE}, 2019, pp. 860--868.

\bibitem{NRP}
M.~Naseer, S.~H. Khan, M.~Hayat, F.~S. Khan, and F.~Porikli, ``A
  self-supervised approach for adversarial robustness,'' in
  \emph{{CVPR}}.\hskip 1em plus 0.5em minus 0.4em\relax Computer Vision
  Foundation / {IEEE}, 2020, pp. 259--268.

\bibitem{RS}
J.~M. Cohen, E.~Rosenfeld, and J.~Z. Kolter, ``Certified adversarial robustness
  via randomized smoothing,'' in \emph{{ICML}}, ser. Proceedings of Machine
  Learning Research, vol.~97.\hskip 1em plus 0.5em minus 0.4em\relax {PMLR},
  2019, pp. 1310--1320.

\bibitem{Ensemble-Adversarial-Training}
F.~Tram{\`{e}}r, A.~Kurakin, N.~Papernot, I.~J. Goodfellow, D.~Boneh, and P.~D.
  McDaniel, ``Ensemble adversarial training: Attacks and defenses,'' in
  \emph{{ICLR} (Poster)}.\hskip 1em plus 0.5em minus 0.4em\relax
  OpenReview.net, 2018.

\bibitem{Torchattacks}
H.~Kim, ``Torchattacks: A pytorch repository for adversarial attacks,''
  \emph{arXiv preprint arXiv:2010.01950}, 2020.

\bibitem{Madry-AT}
A.~Madry, A.~Makelov, L.~Schmidt, D.~Tsipras, and A.~Vladu, ``Towards deep
  learning models resistant to adversarial attacks,'' in \emph{{ICLR}
  (Poster)}.\hskip 1em plus 0.5em minus 0.4em\relax OpenReview.net, 2018.

\bibitem{CW}
N.~Carlini and D.~A. Wagner, ``Towards evaluating the robustness of neural
  networks,'' in \emph{{IEEE} Symposium on Security and Privacy}.\hskip 1em
  plus 0.5em minus 0.4em\relax {IEEE} Computer Society, 2017, pp. 39--57.

\bibitem{APGD}
F.~Croce and M.~Hein, ``Reliable evaluation of adversarial robustness with an
  ensemble of diverse parameter-free attacks,'' in \emph{{ICML}}, ser.
  Proceedings of Machine Learning Research, vol. 119.\hskip 1em plus 0.5em
  minus 0.4em\relax {PMLR}, 2020, pp. 2206--2216.

\bibitem{ZOO}
P.~Chen, H.~Zhang, Y.~Sharma, J.~Yi, and C.~Hsieh, ``{ZOO:} zeroth order
  optimization based black-box attacks to deep neural networks without training
  substitute models,'' in \emph{AISec@CCS}.\hskip 1em plus 0.5em minus
  0.4em\relax {ACM}, 2017, pp. 15--26.

\bibitem{Square}
M.~Andriushchenko, F.~Croce, N.~Flammarion, and M.~Hein, ``Square attack: {A}
  query-efficient black-box adversarial attack via random search,'' in
  \emph{{ECCV} {(23)}}, ser. Lecture Notes in Computer Science, vol.
  12368.\hskip 1em plus 0.5em minus 0.4em\relax Springer, 2020, pp. 484--501.

\bibitem{Free-AT}
A.~Shafahi, M.~Najibi, A.~Ghiasi, Z.~Xu, J.~P. Dickerson, C.~Studer, L.~S.
  Davis, G.~Taylor, and T.~Goldstein, ``Adversarial training for free!'' in
  \emph{NeurIPS}, 2019, pp. 3353--3364.

\bibitem{Fast-AT}
E.~Wong, L.~Rice, and J.~Z. Kolter, ``Fast is better than free: Revisiting
  adversarial training,'' in \emph{{ICLR}}.\hskip 1em plus 0.5em minus
  0.4em\relax OpenReview.net, 2020.

\bibitem{Trades}
H.~Zhang, Y.~Yu, J.~Jiao, E.~P. Xing, L.~E. Ghaoui, and M.~I. Jordan,
  ``Theoretically principled trade-off between robustness and accuracy,'' in
  \emph{{ICML}}, ser. Proceedings of Machine Learning Research, vol.~97.\hskip
  1em plus 0.5em minus 0.4em\relax {PMLR}, 2019, pp. 7472--7482.

\end{thebibliography}

\end{document}